\documentclass{stars-techreport}

\usepackage{titlesec}
\usepackage{printlen}
\usepackage{graphicx}
\graphicspath{{figs/}}

\usepackage{amsmath}
\usepackage{amssymb}
\usepackage{amsfonts}
\usepackage{amstext}
\usepackage{amsthm}
\usepackage{mathrsfs}
\usepackage{bm}
\usepackage{cancel}
\usepackage{mathtools}
\allowdisplaybreaks

\usepackage{float}

\floatstyle{plain}
\newfloat{infobox}{thb}{none}

\usepackage[hang,flushmargin]{footmisc}

\usepackage{booktabs}

\usepackage{mdframed}
\mdfsetup{%
  backgroundcolor=gray!12,
  innerleftmargin=10pt,
  innertopmargin=10pt,
  innerrightmargin=10pt,
  innerbottommargin=10pt,
  skipabove=10pt
}

\usepackage{tikz}
\usepackage{tikz-3dplot}
\usetikzlibrary{%
  3d,
  calc,
  arrows.meta,
  perspective,
  positioning,
  shapes.arrows,
  decorations.pathmorphing,
  decorations.pathreplacing%
}

\usepackage[
  backend=biber,
  giveninits=true,
  hyperref=true,
  style=ieee-caps]%
{biblatex}
\usepackage{aliascnt}
\usepackage[noabbrev]{cleveref}
\crefname{equation}{}{}

\crefname{appendix}{appendix}{appendices}
\Crefname{appendix}{Appendix}{Appendices}

\newtheorem{theorem}{Theorem}[section]

\newaliascnt{proposition}{theorem}

\aliascntresetthe{proposition}

\newaliascnt{lemma}{theorem}

\aliascntresetthe{lemma}

\newaliascnt{corollary}{theorem}
\newtheorem{corollary}[corollary]{Corollary}
\aliascntresetthe{corollary}

\crefname{theorem}{theorem}{theorems}
\Crefname{theorem}{Theorem}{Theorems}

\crefname{proposition}{proposition}{propositions}
\Crefname{proposition}{Proposition}{Propositions}

\crefname{lemma}{lemma}{lemmas}
\Crefname{lemma}{Lemma}{Lemmas}

\crefname{corollary}{corollary}{corollaries}
\Crefname{corollary}{Corollary}{Corollaries}

\theoremstyle{definition}
\newtheorem{definition}[theorem]{Definition}

\crefname{remark}{remark}{remarks}
\Crefname{remark}{Remark}{Remarks}

\usepackage{xparse}

\let\originalleft\left
\let\originalright\right
\renewcommand{\left}{\mathopen{}\mathclose\bgroup\originalleft}
\renewcommand{\right}{\aftergroup\egroup\originalright}

\NewDocumentCommand\Real{}{ \mathbb{R} }

\NewDocumentCommand\bbm{}{ \begin{bmatrix} } 
\NewDocumentCommand\ebm{}{ \end{bmatrix} }   
\NewDocumentCommand\T{}{\mathsf{T}}          

\NewDocumentCommand\Vector{m}{ \boldsymbol{\mathbf{#1}} }

\NewDocumentCommand\Matrix{m}{ \bm{\mathbf{#1}} }

\NewDocumentCommand\Transpose{m}{ \left.{#1}\right.^\T }

\NewDocumentCommand\Inv{m}{{#1}^{-1}}

\NewDocumentCommand\Determinant{m}{ \mathrm{det}\left(#1\right) }
\NewDocumentCommand\Norm{m}{ \left\Vert#1\right\Vert }

\NewDocumentCommand\Ker{m}{ \mathrm{ker}\left(#1\right) }
\NewDocumentCommand\Rank{m}{ \mathrm{rank}\left(#1\right) }
\NewDocumentCommand\Span{m}{ \mathrm{span}\left\{#1\right\} }

\NewDocumentCommand\Zero{}{ \Matrix{0} }
\NewDocumentCommand\Identity{}{ \Matrix{I} }

\NewDocumentCommand\LieGroupSO{m}{ \mathrm{SO}(#1) }

\NewDocumentCommand\LieGroupSGal{m}{ \mathrm{SGal}(#1) }
\NewDocumentCommand\LieAlgebraSGal{m}{ \mathfrak{sgal}(#1) }

\NewDocumentCommand\Wedge{m}{\left(#1\right)^\wedge}
\NewDocumentCommand\Vee{m}{\left(#1\right)^\vee}

\NewDocumentCommand\Matlog{m}{\mathrm{ln}\left(#1\right)}

\NewDocumentCommand\Matexp{m}{\exp\left(#1\right)}

\NewDocumentCommand\Defined{}{ \triangleq }

\usepackage{scalerel}

\NewDocumentCommand\<{}{\mspace{1mu}}

\makeatletter
\newcommand{\vast}{\bBigg@{3}}
\newcommand{\Vast}{\bBigg@{4}}
\makeatother

\newcommand{\uw}{\Vector{u}^{\wedge}}

\newcommand{\alttau}{\ensuremath{\bar{\tau}}}

\NewDocumentCommand\StateVector{}{ \Vector{\mathcal{X}} }

\NewDocumentCommand\SmallStateVector{}{\scaleobj{0.85}{\StateVector}}

\STARStitle{On the Identifiability of Aided Inertial Navigation\\ Under Measurement Delays: A Geometric Approach}
\STARSdate{August 1, 2026}        %
\STARSauthor{Jonathan Kelly}      %
\STARSidentifier{STARS-2026-001}  %
\STARSmajorrev{1}                 %
\STARSminorrev{11}                %

\begin{document}
\maketitle

\begin{abstract}
In aided inertial navigation, measurements from different sensors are often subject to unknown relative time delays.
Consider a single aiding sensor whose measurements have an unknown but constant delay relative to the inertial-measurement data stream. We study the identifiability of the delay and the initial navigation state parameterizing the trajectory.
Identifiability depends on both the temporal structure of the aiding measurements and the form of the trajectory.
Using the special Galilean group, we determine the minimal number and type of aiding measurements needed to recover the delay and the navigation state.
We also characterize a class of \emph{uninformative} trajectories, for which the delayed measurement model admits a continuous symmetry that prevents unique delay-and-state recovery.
We show that each such trajectory is generated by a constant element of the Galilean Lie algebra, and connect this result to the familiar linearized, Jacobian-based analysis.
\end{abstract}

\section{Introduction}
\label{sec:introduction}

An aided inertial navigation system (INS) fuses inertial measurements with data from additional sensors to estimate the state of a moving vehicle, typically including its position, orientation, and velocity.
The inertial measurement unit (IMU) provides high-rate signals that drive the process model, while lower-rate aiding sensors, such as GNSS receivers, supply observations that correct drift to maintain bounded estimation error over longer time horizons.

In practice, measurements from different sensors are often subject to relative delays.
These delays arise from sensing, processing, and communication latencies, and are commonly modelled as constant but unknown~\cite{2014_Kelly_Determining,2014_Li_Online}.
If not accounted for in the estimator, they can degrade accuracy and, in adverse cases, lead to divergence of the navigation solution.
For a single aiding sensor, the estimation problem therefore involves recovering both the unknown delay and the navigation state from delayed measurements.

Some existing approaches estimate the measurement delay by adding it to the estimator state vector and applying a standard filtering algorithm, such as the extended Kalman filter (EKF).
However, these formulations are often adopted without first establishing whether the augmented system is identifiable (or observable). 
There is an implicit assumption that isolated delayed measurements provide enough information to recover the unknowns.
Our prior work~\cite{2021_Kelly_Question} showed that such augmented-state recursive filters have fundamental structural issues that lead to inconsistent estimates and poorer performance.

In this report, we study the identifiability of aided inertial navigation under measurement time delays.
Prior work on excitation requirements for aided navigation with delayed sensor measurements~\cite{2019_Yang_Degenerate,2023_Yang_Online} is tangentially related: it treats a different measurement model and identifies specific problematic cases through a linearized analysis.
Our focus is different, and more general: we seek to determine the \emph{minimal} sets of measurements required for identifiability, and to characterize the classes of motion for which the delay and initial navigation state can, or cannot, be uniquely recovered.
We show that delayed aided navigation is identifiable only when the available measurements \emph{over time} are sufficiently informative to constrain both the delay and the state.
The analysis is formulated on the special Galilean group, which provides an appropriate state-space framework.
We first analyze the one-dimensional case, where the geometry of the problem is easiest to see, and then extend the analysis to the three-dimensional setting.

A central contribution of this report is to show that the delayed measurement model admits a continuous symmetry, in which changes to the delay can be compensated by changes to the initial navigation state without altering the measurement history.
This symmetry defines a class of \emph{uninformative trajectories} for which the delay and initial state cannot be uniquely recovered.
We characterize this class geometrically, in terms of the Galilean Lie algebra, and connect the exact nonlinear results to the familiar Jacobian-based local tests.
Related ideas appear in the symmetry-based observability analysis developed by Martinelli for non-delayed systems~\cite{2011_Martinelli_State}, but our approach is built around delayed measurement histories.

The remainder of the report is organized as follows.
We begin with some preliminaries in \Cref{sec:preliminaries}.
In \Cref{sec:SGal1_group}, we introduce the one-dimensional special Galilean group $\LieGroupSGal{1}$; identifiability on $\LieGroupSGal{1}$ is analyzed in \Cref{sec:back_to_basics}.
We then consider the three-dimensional special Galilean group $\LieGroupSGal{3}$ in \Cref{sec:SGal3_group} and extend our analysis to the full delayed aided-navigation problem in \Cref{sec:aided_identifiability}.
We conclude in \Cref{sec:conclusion} with a summary of our results.

\section{Preliminaries}
\label{sec:preliminaries}

To begin, we establish a small amount of notation and the main concepts used throughout the report.
Lowercase Latin and Greek letters (e.g., $a$ and $\alpha$) denote scalars, while boldface lower- and uppercase Latin letters (e.g., $\Vector{x}$ and $\Matrix{X}$) denote vectors and matrices, respectively.
We write the $n \times n$ identity matrix as $\Identity_n$ and the $m \times n$ matrix of zeros as $\Zero_{m \times n}$; when the dimensions are clear from context, we often omit the subscripts.
Additional notation is introduced as needed.

\subsection{Identifiability and Observability}
\label{subsec:identifiability}

The main question in this report is whether an unknown delay and other relevant unknown parameters can be uniquely determined, or \emph{identified}, from discrete output measurements over a finite interval.
The parameters typically include the initial navigation state (position, velocity, and orientation) and in some cases bias terms.

The topic of identifiability is closely related to observability.
Roughly speaking, observability concerns recovery of a time-varying state from measured outputs, whereas identifiability concerns recovery of unknown but fixed parameters.
In our setting, the two viewpoints are closely connected: we assume the input history is known, so the delay and initial navigation state can be treated together as parameters to be inferred from measurements.
For this reason, we primarily use the language of identifiability; see, for example,~\cite{2015_Chatzis_Observability} for more details.
We refer to the initial navigation state as the \emph{initial condition} in the sequel.

Several practical points are worth stating at the outset.
First, we study the \emph{noise-free} problem, as is typical in the identifiability literature.
Second, although the delay may take any real value, in practice we are interested in causal delays.
Third, the input and state history must be retained long enough for the delayed outputs to depend on a portion of the stored history.
This basic overlap requirement, and its role in making the estimation problem well posed, was made explicit in our prior work~\cite{2021_Kelly_Question}.

\newpage
We now formalize the notion of local identifiability used throughout the report.
Consider a nonlinear control system of the form
\begin{equation}
\label{eqn:general_delayed_system}
\begin{gathered}
\dot{\Matrix{X}}(t)
=
f\bigl(
\Matrix{X}(t), \Vector{u}(t), \Vector{\theta}
\bigr),
\quad
\Matrix{X}(0) = \Matrix{X}_0 \in \mathcal{M}, \\[1mm]
\Matrix{Y}(t)
=
h\bigl(\Matrix{X}(t - \tau),\,\tau\bigr),
\quad
t \ge \tau,
\end{gathered}
\end{equation}
where $\mathcal{M}$ is an $m$-dimensional smooth manifold, $\Matrix{X}(t) \in \mathcal{M}$ is the state, $\Vector{u}(t) \in \mathcal{U} \subseteq \Real^p$ is a known control input, and $\Vector{\theta} \in \Theta \subseteq \Real^q$ is a vector of additional unknown parameters.
The map $f : \mathcal{M} \times \mathcal{U}\times\Theta \to T\mathcal{M}$ is a smooth vector field, with $f(\Matrix{X},\Vector{u},\Vector{\theta}) \in T_{\Matrix{X}}\mathcal{M}$.
The output $\Matrix{Y}(t) \in \mathcal{N}$ lies on an $\ell$-dimensional smooth manifold $\mathcal{N}$, the map $h:\mathcal{M} \times \Real_{\geq 0} \to \mathcal{N}$ is smooth, and $\tau \geq 0$ is an unknown constant delay.

We collect all the unknowns into a tuple,
\begin{equation*}
\StateVector
\Defined
\left(\Matrix{X}_0,\,\tau,\,\Vector{\theta}\right)
\in \mathcal{P},
\end{equation*}
where $\Matrix{X}_0 = \Matrix{X}(0)$ is the initial condition and $\mathcal{P}$ is the admissible parameter space.
We use $\SmallStateVector\in\Real^d$ for a local-coordinate vector corresponding to $\StateVector$, with $d = \dim(\mathcal{P})$.

We assume the system trajectory is defined for all $t \geq 0$, while measurements are collected over a finite observation window.
For identifiability, output histories should be compared on an interval that is valid for every admissible delay.
Let $\tau_u$ be an upper bound on the admissible delays. We restrict attention to observation intervals beginning no earlier than $\tau_u$, so that the delayed state is well defined for every such delay.
Local identifiability is then determined by whether distinct parameter values can produce the same output history (i.e., measurements) over that interval.

Let $[t_s,\,t_s + T]$ be such an observation interval, with $t_s \geq \tau_u$.
For a fixed known input history sufficient to propagate the state from the initial time to $t_s + T$, each parameter tuple $\StateVector \in \mathcal{P}$ determines an output history over $[t_s,\, t_s + T]$.
Changing $\tau$ shifts where along the trajectory the state is evaluated, making delay identifiability a trajectory-dependent problem.

\begin{definition}[Indistinguishability]
Let $\mathcal{T} \subseteq [t_s,\, t_s + T]$ be a set of observation times.
The parameter tuples $\StateVector^\circ, \StateVector^\star \in \mathcal{P}$ are said to be \emph{indistinguishable} on $\mathcal{T}$ if, for the same known input history, they produce identical outputs at every time in $\mathcal{T}$.
The corresponding delayed states may occur at different times along the trajectory, since the delays in $\StateVector^\circ$ and $\StateVector^\star$ need not be equal.
Otherwise, the parameter tuples are said to be \emph{distinguishable} on $\mathcal{T}$.
\end{definition}

\begin{definition}[Local identifiability]
A parameter tuple $\StateVector^\star \in \mathcal{P}$ is \emph{locally identifiable} on a set of observation times $\mathcal{T} \subseteq [t_s,\, t_s + T]$ if there exists a neighbourhood $\mathcal{V}$ of $\StateVector^\star$ such that $\StateVector^\star$ is distinguishable on $\mathcal{T}$ from every $\StateVector \in \mathcal{V}$ with $\StateVector \neq \StateVector^\star$.
Equivalently, any $\StateVector \in \mathcal{V}$ that produces the same outputs as $\StateVector^\star$ at every time in $\mathcal{T}$ must satisfy
\begin{equation*}
\StateVector = \StateVector^\star.
\end{equation*}
\end{definition}

This local concept is the one used throughout the report.
It is weaker than global identifiability, which would require uniqueness over the entire parameter space, but it is the usual notion for local estimation and linearization-based analyses~\cite{1970_Bellman_Structural,1976_Grewal_Identifiability}.

A standard local test is obtained by linearizing the map from the unknown parameters to the outputs at a finite set of measurement times in $[t_s,\, t_s + T]$.
If the corresponding Jacobian has full column rank at $\StateVector^\star$, then the parameter-to-output map is locally injective, and local identifiability follows.
Conversely, if the Jacobian is rank deficient, then there exist infinitesimal perturbation directions that leave the outputs unchanged to first order.
This does not, by itself, prove full nonlinear unidentifiability, but it does indicate a loss of local information and motivates the nullspace analysis carried out later on.

\subsection{Continuous Symmetries}
\label{subsec:symmetries}

A central idea in our analysis is that unidentifiability can arise from symmetry.
Here, a symmetry is a transformation of the parameters that leaves the outputs unchanged.
When such a transformation exists, distinct parameter values produce the same measurement(s), so the unknowns cannot be uniquely determined.

The most important case for our purposes is a continuous symmetry.
This occurs when the delay and initial condition can be varied together without changing the outputs over the observation window.
Geometrically, the true parameters are then not isolated in parameter space, but lie on a local family of indistinguishable solutions.
Related ideas appear in nonlinear observability, including Martinelli's treatment of continuous symmetries for non-delayed systems~\cite{2011_Martinelli_State}.
The setting considered here is similar, but temporal: the relevant symmetries act on delayed measurement histories over an interval, rather than on measurements at a single instant.
We first give a trajectory-level result, which will be useful later when the symmetry holds over the entire observation window.

\begin{theorem}[Symmetry implies unidentifiability]
\label{thm:symmetry_unidentifiable_continuous}
Let $\StateVector^\star \in \mathcal{P}$.
Suppose there exists $\epsilon > 0$ and a smooth one-parameter family of transformations $\mathcal{S}_\alpha : \mathcal{P} \to \mathcal{P}$, $\alpha \in (-\epsilon, \epsilon)$, such that $\mathcal{S}_0(\StateVector)=\StateVector$ for all $\StateVector$ in a neighbourhood of $\StateVector^\star$.
Suppose further that, for every $\alpha \in (-\epsilon, \epsilon)$, the parameter tuples $\StateVector^\star$ and $\mathcal{S}_\alpha(\StateVector^\star)$ are indistinguishable on $[t_s,\, t_s + T]$.
If $\mathcal{S}_\alpha(\StateVector^\star) \neq \StateVector^\star$ for all $\alpha \in (-\epsilon, \epsilon) \setminus \{0\}$, then $\StateVector^\star$ is not locally identifiable on $[t_s,\, t_s + T]$.
\end{theorem}

\begin{proof}
For every $\alpha \in (-\epsilon, \epsilon)$, the tuples $\StateVector^\star$ and $\mathcal{S}_\alpha(\StateVector^\star)$ are indistinguishable on $[t_s,\, t_s + T]$.
By continuity, $\mathcal{S}_\alpha(\StateVector^\star) \to \StateVector^\star$ as $\alpha \to 0$, so every neighbourhood of $\StateVector^\star$ contains some $\mathcal{S}_\alpha(\StateVector^\star) \neq \StateVector^\star$ producing the same output history.
Hence the parameter-to-output map is not locally injective at $\StateVector^\star$, and $\StateVector^\star$ is not locally identifiable.
\end{proof}

In typical aided navigation systems, measurements are available only at discrete times.
The following corollary gives the corresponding result for a finite measurement set.

\begin{corollary}[Discrete-measurement symmetry implies unidentifiability]
\label{cor:symmetry_unidentifiable_discrete}
Let $\StateVector^\star \in \mathcal{P}$, and let $t_1,\dots,t_n$ denote the measurement times.
Suppose there exists $\epsilon > 0$ and a smooth one-parameter family of transformations $\mathcal{S}_\alpha : \mathcal{P} \to \mathcal{P}$, $\alpha \in (-\epsilon, \epsilon)$, such that $\mathcal{S}_0(\StateVector) = \StateVector$ for all $\StateVector$ in a neighbourhood of $\StateVector^\star$.
Suppose further that, for every $\alpha \in (-\epsilon, \epsilon)$, the parameter tuples $\StateVector^\star$ and $\mathcal{S}_\alpha(\StateVector^\star)$ are indistinguishable on $\{t_1, \dots, t_n\}$.
If $\mathcal{S}_\alpha(\StateVector^\star) \neq \StateVector^\star$ for all $\alpha \in (-\epsilon, \epsilon) \setminus \{0\}$, then $\StateVector^\star$ is not locally identifiable from the discrete measurements.
\end{corollary}

\begin{proof}
The proof is identical to that of \Cref{thm:symmetry_unidentifiable_continuous}, with the output history replaced by the finite collection of outputs at $\{t_1, \dots, t_n\}$.
\end{proof}

The relevant symmetry arises when a time shift along the trajectory can be absorbed into a corresponding change in the initial condition.
Different delay--initial-condition pairs then produce the same measurement history, and the delay is not identifiable.
Below, we use two complementary tools:
\begin{enumerate}
\setlength{\itemsep}{2pt}
\item exact nonlinear symmetries of the measurement map, which imply true unidentifiability; and
\item infinitesimal symmetry directions, obtained from the nullspace of the Jacobian, which characterize local loss of information.
\end{enumerate}
The first proves directly that certain trajectories or measurement configurations are uninformative, while the second gives a standard local test, with a corresponding rank deficiency in the linearized measurement model.

\section{The Galilean Group $\LieGroupSGal{1}$}
\label{sec:SGal1_group}

We begin with the simplest Galilean group structure, $\LieGroupSGal{1}$, to make the symmetry argument easier to see.
In one spatial dimension, the only proper rotation is the identity.
Consequently, $\LieGroupSGal{1}$ consists only of Galilean boosts, spatial translations, and time translations. 
Elements of the group can be written as $3 \times 3$ matrices,
\begin{equation}
\label{eqn:SGal1_definition}
\LieGroupSGal{1} 
\Defined
\begin{Bmatrix}
\Matrix{X} =
\bbm 
1 & v & r \\ 
0 & 1 & \eta \\
0 & 0 & 1
\ebm
\in \Real^{3 \times 3}
\ \vast\vert \ 
v,\, r,\, \eta \in \Real
\end{Bmatrix},
\end{equation}
where $v \in \Real$ is the velocity (boost), $r \in \Real$ is the spatial translation, and $\eta \in \Real$ is the time translation.
The group operation is matrix multiplication.
The inverse of $\Matrix{X}$ is
\begin{equation*}
\Inv{\Matrix{X}} =
\bbm
1 & -v & v\<\eta - r \\ 
0 &  1 & -\eta \\
0 &  0 & 1
\ebm,
\end{equation*}
such that $\Matrix{X}\Inv{\Matrix{X}} = \Identity_{3}$.

A group element may be interpreted as a transformation between inertial frames related by constant relative motion, and acts on spacetime coordinates $(x, t)$ according to
\begin{equation*}
(x, t) \mapsto (x + v\<t + r,\, t + \eta).
\end{equation*}
Using homogeneous coordinates $\bbm x\! & t\! & 1 \ebm^\T$, this action can be written as
\begin{equation*}
\bbm
x \\ t \\ 1	
\ebm
\mapsto
\Matrix{X}
\bbm
x \\ t \\ 1
\ebm
=
\bbm
1 & v & r \\
0 & 1 & \eta \\
0 & 0 & 1
\ebm
\bbm
x \\ t \\ 1
\ebm.
\end{equation*}
For two group elements
\begin{equation*}
\Matrix{X}_1 =
\bbm
1 & v_1 & r_1 \\
0 & 1 & \eta_1 \\
0 & 0 & 1
\ebm,
\quad
\Matrix{X}_2 =
\bbm
1 & v_2 & r_2 \\
0 & 1 & \eta_2 \\
0 & 0 & 1
\ebm,
\end{equation*}
their composition is
\begin{equation*}
\Matrix{X}_1\Matrix{X}_2
=
\bbm
1 & v_1 + v_2 & r_1 + r_2 + v_1 \eta_2 \\
0 & 1 & \eta_1 + \eta_2 \\
0 & 0 & 1
\ebm.
\end{equation*}
The term $v_1 \eta_2$ reflects the coupling between boosts and time translations.

\subsection{The Lie Algebra $\LieAlgebraSGal{1}$}

The Lie algebra $\LieAlgebraSGal{1}$ consists of matrices of the form
\begin{equation}
\label{eqn:sgal1_algebra_definition}
\LieAlgebraSGal{1}
\Defined
\begin{Bmatrix}
\Matrix{\Xi} =
\bbm
0 & \nu & \rho \\
0 & 0   & \iota \\
0 & 0   & 0
\ebm
\in \Real^{3 \times 3}
\ \vast\vert \
\nu,\, \rho,\, \iota \in \Real
\end{Bmatrix}.
\end{equation}
The scalars $\nu$, $\rho$, and $\iota$ are the boost, spatial-translation, and time-translation components of a tangent vector at the identity.
A natural basis for $\LieAlgebraSGal{1}$ is given by
\begin{equation}
\label{eqn:sgal1_basis}
\Matrix{B} =
\bbm
0 & 1 & 0 \\
0 & 0 & 0 \\
0 & 0 & 0
\ebm,
\quad
\Matrix{P} =
\bbm
0 & 0 & 1 \\
0 & 0 & 0 \\
0 & 0 & 0
\ebm,
\quad
\Matrix{K} =
\bbm
0 & 0 & 0 \\
0 & 0 & 1 \\
0 & 0 & 0
\ebm.
\end{equation}
The Lie bracket is defined by the matrix commutator, and the only nontrivial relation is
\begin{equation*}
[\Matrix{B},\< \Matrix{K}] = \Matrix{P},
\quad
[\Matrix{B},\< \Matrix{P}] = \Matrix{0},
\quad
[\Matrix{K},\< \Matrix{P}] = \Matrix{0}.
\end{equation*}
We introduce the wedge operator $\Wedge{\cdot}$ as a mapping $\Real^{3} \rightarrow \LieAlgebraSGal{1}$,
\begin{equation}
\label{eqn:sgal1_wedge}
\Vector{\xi}^{\wedge}
=
\bbm
\rho \\
\nu \\
\iota
\ebm^{\wedge}
\Defined
\bbm
0 & \nu & \rho \\
0 & 0   & \iota \\
0 & 0   & 0
\ebm,
\end{equation}
where $\Vector{\xi} = \bbm \rho &\! \nu &\! \iota \ebm^{\T} \in \Real^3$.
The corresponding vee operator $\Vee{\cdot}$ is defined by
\begin{equation*}
\Vector{\xi}^{\wedge} = \Matrix{\Xi}
\quad \Longleftrightarrow \quad
\Matrix{\Xi}^{\vee} = \Vector{\xi}.
\end{equation*}

\subsection{The Exponential Map}
\label{subsec:SGal1_exp_log_maps}

Having defined the Lie algebra $\LieAlgebraSGal{1}$, we now give the exponential map from $\LieAlgebraSGal{1}$ to $\LieGroupSGal{1}$.
Let $\Vector{\xi} =\smash{\bbm \rho &\! \nu &\! \iota \ebm^{\raisebox{-0.15ex}{$\scriptstyle\T$}}} \in \Real^3$.
The exponential map is defined by the matrix exponential,
\begin{equation*}
\Matexp{\Vector{\xi}^{\wedge}}
=
\sum_{n=0}^{\infty}
\frac{1}{n!}
\left(\Vector{\xi}^{\wedge}\right)^n.
\end{equation*}
Since $\left(\Vector{\xi}^{\wedge}\right)^3 = \Matrix{0}$, the series truncates after the quadratic term, yielding
\begin{equation}
\label{eqn:SGal1_exp_closed}
\Matexp{\Vector{\xi}^{\wedge}}
=
\bbm
1 & \nu & \rho + \displaystyle\frac{1}{2}\nu \iota \\
0 & 1   & \iota \\
0 & 0   & 1
\ebm.
\end{equation}
The term $\frac{1}{2}\nu\iota$ arises from the noncommutativity of boosts and time translations, as reflected in the Lie bracket $[\Matrix{B},\< \Matrix{K}] = \Matrix{P}$.

\section{Back to Basics: An Analysis on $\LieGroupSGal{1}$}
\label{sec:back_to_basics}

We begin by examining a simplified one-dimensional model that captures the structure of the identifiability problem.
Our aim is to establish a common approach for analyzing identifiability under different measurement sequences.
We proceed in two stages.
First, we study the constraints induced by different measurement combinations.
The constraint-based analysis gives the excitation conditions a trajectory must satisfy for the delay and initial condition to be identifiable, while also characterizing which changes in the delay and initial condition leave the measurement history unchanged.
We then revisit the same cases using the standard Jacobian-based approach, determining local identifiability from the rank of the measurement Jacobian.

We write continuous-time quantities as functions of time, such as $\Matrix{X}(t)$, $r(t)$, and $v(t)$.
Subscripts denote evaluation at discrete time instants, so that $\Matrix{X}_k \Defined \Matrix{X}(t_k)$, $r_k \Defined r(t_k)$, and $v_k \Defined v(t_k)$, where $k$ is an index variable.

\subsection{A Simple Aided Inertial Navigation System}
\label{subsec:SGal1_system}

Consider a one-dimensional aided INS driven by the acceleration input $a(t)$, with gravity neglected to keep the model as simple as possible.
Without loss of generality, we take the initial time to be $t = 0$, consistent with the general delayed system in \Cref{eqn:general_delayed_system}.
We represent the navigation state on $\LieGroupSGal{1}$ in matrix form as
\begin{equation}
\label{eqn:SGal1_state}
\Matrix{X}(t)
=
\bbm
1 & v(t) & r(t) \\
0 &    1 &    t \\
0 &    0 &    1
\ebm,
\end{equation}
where $r(t) \in \Real$ is the position and $v(t) \in \Real$ is the velocity of the vehicle (navigation) frame relative to a fixed inertial reference frame.\footnote{To avoid clutter, we omit the frame identifiers throughout.}
Thus, $\Matrix{X}(t)$ is a time-parameterized curve on $\LieGroupSGal{1}$.
The initial condition is
\begin{equation}
\label{eqn:SGal1_initial_state}
\Matrix{X}_0
=
\bbm
1 & v_0 & r_0 \\
0 &   1 &   0 \\
0 &   0 &   1
\ebm
\in \LieGroupSGal{1},
\end{equation}
where $r_0 = r(0)$ and $v_0 = v(0)$ are unknown. Since the initial time is $t = 0$, $\Matrix{X}_0$ lies in the \emph{isochronous} subgroup of $\LieGroupSGal{1}$, whose elements have zero time component~\cite{2002_Bhand_Rigid,2023_Kelly_Galilean}.
That is, the navigation-frame and inertial-frame clocks share a common origin.

The continuous-time kinematics are
\begin{equation}
\label{eqn:1D_kinematics}
\dot{v}(t) = a(t),
\quad
\dot{r}(t) = v(t),
\end{equation}
where the input $a(t)$ is assumed known.
Define the integrated input quantities\footnote{We use $(\breve{\cdot})$ to indicate a quantity obtained by integrating the known input history.}
\begin{equation*}
\breve{v}(t)
\Defined
\int_{0}^{t} a(u)\< du,
\quad
\breve{r}(t)
\Defined
\int_{0}^{t} \breve{v}(u)\< du
=
\int_{0}^{t} (t - u)\< a(u)\< du.
\end{equation*}
The solution of \Cref{eqn:1D_kinematics} is then
\begin{equation*}
v(t) = v_0 + \breve{v}(t),
\quad
r(t) = r_0 + v_0\<t + \breve{r}(t).
\end{equation*}
Equivalently, the state trajectory may be written in group form as
\begin{equation*}
\Matrix{X}(t)
=
\Matrix{X}_0\<\Matrix{\Phi}(t),
\end{equation*}
where the state-transition matrix $\Matrix{\Phi}(t)$ is
\begin{equation}
\label{eqn:SGal1_phi_matrix}
\Matrix{\Phi}(t)
=
\bbm
1 & \breve{v}(t) & \breve{r}(t) \\
0 & 1 & t \\
0 & 0 & 1
\ebm
\in \LieGroupSGal{1}.
\end{equation}

Aiding measurements arrive at discrete times $t_k$, for $k = 1, \dots, n$.
A measurement received at time $t_k$ is of the navigation state at the earlier time $t_k - \tau$, where $\tau$ is the unknown constant delay.
Because time is a coordinate of the Galilean group element, time shifts can be expressed by the \emph{time-translation generator}
\begin{equation}
\label{eqn:sgal1_time_translation_generator}
\Vector{\xi}_d
\Defined
\bbm
\Zero_{1 \times 2} & 1
\ebm^\T
\in \Real^{3}.
\end{equation}
With our ordering of the Lie algebra coordinates, $\Vector{\xi}_d$ has only a time component.
Let
\begin{equation}
\label{eqn:SGal1_pure_time_translation}
\Matrix{T}(\beta)
\Defined
\Matexp{\beta\<\Vector{\xi}_d^{\wedge}}
\end{equation}
denote the corresponding time translation, where $\Vector{\xi}_d^{\wedge} = \Matrix{K}$ from \Cref{eqn:sgal1_basis}.
Under left multiplication, $\Matrix{T}(\beta)$ shifts the group time coordinate by $\beta$ while leaving the other components unchanged; this is the operation used below.

The most informative measurement is of the full delayed navigation state (i.e., position and velocity).
The position and velocity are recorded at the delayed time $t_k - \tau$, but the measurement does not arrive until time $t_k$.
A left time translation $\Matrix{T}(\tau)$ advances the group time coordinate from $t_k - \tau$ to $t_k$; we therefore define
\begin{equation}
\label{eqn:SGal1_delayed_measurement}
\Matrix{Y}_k
\Defined
\Matrix{T}(\tau)\<\Matrix{X}(t_k - \tau)
=
\Matrix{T}(\tau)\<\Matrix{X}_0\<\Matrix{\Phi}(t_k - \tau)
=
\bbm
1 & v(t_k - \tau) & r(t_k - \tau) \\
0 & 1 & t_k \\
0 & 0 & 1
\ebm,
\quad
k=1, \dots, n.
\end{equation}
The left translation acts on the measurement time only, so the initial condition $\Matrix{X}_0$ retains the zero time coordinate and remains in the isochronous subgroup.

For the linearized analysis later, we will need to differentiate only the position and velocity components of the measurement with respect to the unknowns.
We therefore define $\pi : \LieGroupSGal{1} \to \Real^2$ to extract the position and velocity entries of a group element,
\begin{equation}
\label{eqn:SGal1_measurement_extracted}
\Vector{y}_k
=
\pi(\Matrix{Y}_k)
\Defined
\bbm
r(t_k - \tau) \\
v(t_k - \tau)
\ebm.
\end{equation}
The unknown parameters are the initial condition $(r_0, v_0)$ and the delay $\tau$, while the functions $\breve{v}(\cdot)$ and $\breve{r}(\cdot)$ are fully determined by the known input history.
The model \Cref{eqn:SGal1_delayed_measurement} allows us to treat a range of measurement cases in a unified way.
In the next section, we consider delayed observations consisting of the full state and of position only.
The central question is always the same: given a finite number of measurements, can the unknown delay $\tau$ and the unknown initial condition $(r_0, v_0)$ be uniquely determined?

\subsection{Identifiability}
\label{subsec:SGal1_identifiability}

Next, we explore how the choice and number of measurements affect the identifiability of the delay and the initial condition.
We proceed incrementally, beginning with a single measurement and extending to multiple measurements over time.
Throughout, we assume that each measurement time $t_k$ is large enough, relative to the candidate delays under consideration, that the delayed evaluation time $t_k - \tau$ lies within the available trajectory history.

\subsubsection{A Single Full-State Measurement}
\label{subsubsec:SGal1_one_full_measurement}

With a single measurement of position and velocity (i.e., a `full-state' measurement in the one-dimensional setting) at a time $t_1$,
\begin{equation*}
\Matrix{Y}_1
=
\Matrix{T}(\tau)\<
\Matrix{X}_0\<
\Matrix{\Phi}(t_1 - \tau),
\end{equation*}
the unknown initial condition cannot be eliminated. In fact, for any
admissible alternative delay $\alttau$ define
\begin{equation*}
\bar{\Matrix{X}}_0
=
\Matrix{T}(\tau - \alttau)\<
\Matrix{X}_0\<
\Matrix{\Phi}(t_1 - \tau)\<
\Matrix{\Phi}(t_1 - \alttau)^{-1}
\end{equation*}
such that
\begin{equation*}
\Matrix{Y}_1
=
\Matrix{T}(\alttau)\<
\bar{\Matrix{X}}_0\<
\Matrix{\Phi}(t_1 - \alttau).
\end{equation*}
That is, the same observation may be explained by different values of $\tau$ together with corresponding changes in the initial condition.
The delay and initial condition are therefore jointly unidentifiable from a single full-state measurement.
In this case, the transformation of $(\Matrix{X}_0,\tau)$ above satisfies the conditions of \Cref{cor:symmetry_unidentifiable_discrete} and defines a continuous symmetry of the measurement model: with $\alpha = \alttau - \tau$, it is a smooth one-parameter family fixing the true parameters at $\alpha = 0$.

\subsubsection{Two Full-State Measurements}
\label{subsubsec:SGal1_two_full_measurements}

Given two full-state measurements at times $t_1$ and $t_2$,
\begin{equation*}
\Matrix{Y}_1
=
\Matrix{T}(\tau)\<
\Matrix{X}_0\<
\Matrix{\Phi}(t_1 - \tau),
\quad
\Matrix{Y}_2
=
\Matrix{T}(\tau)\<
\Matrix{X}_0\<
\Matrix{\Phi}(t_2 - \tau),
\end{equation*}
the unknown initial condition \emph{can} be eliminated. In particular,
\begin{equation*}
\Matrix{Y}_1^{-1}\Matrix{Y}_2
=
\Matrix{\Phi}(t_1 - \tau)^{-1}
\Matrix{\Phi}(t_2 - \tau).
\end{equation*}
That is, the measurements induce a constraint on the delay $\tau$ that is
independent of $\Matrix{X}_0$.
Accordingly, identifiability of the delay is governed by the injectivity
of the map
\begin{equation*}
\tau
\mapsto
\Matrix{\Phi}(t_1 - \tau)^{-1}
\Matrix{\Phi}(t_2 - \tau).
\end{equation*}
If this map is locally injective in a neighbourhood of the true delay, then two delayed full-state measurements are sufficient to identify the delay, and hence the initial condition.
Consequently, identifiability depends on the properties of the trajectory $\Matrix{\Phi}(t)$.

For $\LieGroupSGal{1}$, this condition can be expressed explicitly in terms of the kinematic variables.
Two delayed velocity measurements satisfy
\begin{equation*}
v_1 = v_0 + \breve{v}(t_1 - \tau),
\quad
v_2 = v_0 + \breve{v}(t_2 - \tau).
\end{equation*}
Subtracting eliminates $v_0$ to yield
\begin{equation*}
v_2 - v_1
=
\breve{v}(t_2 - \tau) - \breve{v}(t_1 - \tau).
\end{equation*}
Hence, the delay is constrained by the integrated input $\breve{v}(t)$.
One sufficient condition for local identifiability is that the scalar map
\begin{equation*}
\tau
\mapsto
\breve{v}(t_2 - \tau)-\breve{v}(t_1 - \tau)
\end{equation*}
is locally injective at the true delay.
A convenient local test is obtained by differentiation:
\begin{equation*}
\frac{d}{d\tau}
\big(\breve{v}(t_2 - \tau) - \breve{v}(t_1 - \tau)\big)
=
a(t_1 - \tau) - a(t_2 - \tau).
\end{equation*}
A nonzero derivative is sufficient for local injectivity.
Therefore, the delay is locally identifiable given two delayed measurements whenever
\begin{equation}
\label{eqn:SGal1_two_measurement_condition}
a(t_2 - \tau)\neq a(t_1 - \tau).
\end{equation}
The delay and the initial velocity are determined from the velocity measurements alone. Once these are known, a single position measurement is sufficient to recover the initial position.

\subsubsection{Two Position Measurements}
\label{subsubsec:SGal1_two_position_measurements}

What if only position measurements are available, as is often the case in practice? If two position measurements are available, then
\begin{equation*}
r_k = r\left(t_k - \tau\right)
=
r_0 + \left(t_k - \tau\right)\<v_0 + \breve{r}\left(t_k - \tau\right),
\quad 
k = 1, 2.
\end{equation*}
Subtracting eliminates $r_0$ but not $v_0$,
\begin{equation*}
r_2 - r_1
=
\left(t_2 - t_1\right)\<v_0 +
\breve{r}(t_2 - \tau) - \breve{r}(t_1 - \tau).
\end{equation*}
The result is a single scalar constraint in the two unknowns $v_0$ and $\tau$.
Since the coefficient of $v_0$ is $t_2 - t_1 \neq 0$, the constraint determines $v_0$ as a smooth function of $\tau$. %
Two delayed position-only measurements are therefore insufficient to uniquely determine the delay and the initial condition.

\subsubsection{Three or More Position Measurements}
\label{subsubsec:SGal1_three_position_measurements}

Consider three position measurements taken at times $t_1, t_2, t_3$, with
\begin{equation*}
r_k = r_0 + (t_k - \tau)\<v_0 + \breve{r}(t_k - \tau),
\quad 
k = 1, 2, 3.
\end{equation*}
Differencing eliminates $r_0$,
\begin{equation*}
r_j - r_i
=
(t_j - t_i)\<v_0 + \breve{r}(t_j - \tau) - \breve{r}(t_i - \tau),
\end{equation*}
and eliminating $v_0$ between any two equations yields a scalar constraint on $\tau$. Define
\begin{equation*}
\begin{split}
\psi(\tau)
\Defined
&~
(t_3 - t_2)\bigl(r_2 - r_1 - \breve{r}(t_2 - \tau) + \breve{r}(t_1 - \tau)\bigr) \\
&~
- (t_2 - t_1)\bigl(r_3 - r_2 - \breve{r}(t_3 - \tau) + \breve{r}(t_2 - \tau)\bigr).
\end{split}
\end{equation*}
The delay $\tau$ must satisfy $\psi(\tau) = 0$, and identifiability reduces to whether this equation has a locally unique solution.
A convenient local test is again obtained by differentiation. Using $\dot{\breve{r}}(t)=\breve{v}(t)$,
\begin{equation}
\label{eqn:SGal1_psi_condition}
\frac{d\<\psi\left(\tau\right)}{d\tau}
=
(t_3 - t_2)\bigl(\breve{v}(t_2 - \tau)-\breve{v}(t_1 - \tau)\bigr) -
(t_2 - t_1)\bigl(\breve{v}(t_3 - \tau)-\breve{v}(t_2 - \tau)\bigr).
\end{equation}
Local identifiability holds whenever $d\mspace{0.0mu}\psi/d\tau \neq 0$.
Therefore, three position measurements generically suffice to determine the delay under sufficiently exciting input histories.

The condition \Cref{eqn:SGal1_psi_condition} depends explicitly on the trajectory through the integrated input $\breve{v}(\cdot)$.
In particular, it requires that the increments of $\breve{v}(\cdot)$ over successive intervals not be proportional to the corresponding time differences.
Additional measurements provide independent constraints on $\tau$.
With more measurements, these constraints tend to enlarge the region over which identifiability holds.
However, identifiability still depends on the trajectory, and no finite number of measurements guarantees identifiability for all possible motions.

\subsection{Measurement Jacobians and Nullspace Structure}
\label{subsec:SGal1_jacobian_nullspace}

We now revisit identifiability of the $\LieGroupSGal{1}$ system from a complementary, linearization-based perspective, examining the Jacobian of the measurement function with respect to the unknown parameters.
Local identifiability is characterized by its rank and nullspace.
We follow the same sequence as in \Cref{subsec:SGal1_identifiability} (i.e., a single measurement, two measurements, etc.).

In this section, we use a parameter vector representation of the initial condition and delay (rather than the group-based representation).
The parameter vector is
\begin{equation}
\label{eqn:SGal1_parameter_vector}
\SmallStateVector
\Defined
\bbm
r_0\! & v_0\! & \tau
\ebm^{\T}.
\end{equation}
Recall that the delayed full-state measurement at time $t_k$ in the vector form of \Cref{eqn:SGal1_measurement_extracted} is
\begin{equation*}
\Vector{y}_k
=
\bbm
r_0 + (t_k - \tau)\< v_0 + \breve{r}(t_k - \tau) \\[0.5mm]
v_0 + \breve{v}(t_k - \tau)
\ebm.
\end{equation*}

\subsubsection{A Single Full-State Measurement}
\label{subsubsec:SGal1_one_full_measurement_jacobian}

Consider a single full-state measurement at time $t_1$.
The Jacobian of the measurement function with respect to $\SmallStateVector$ is\footnote{Superscripts in parentheses denote measured quantities (e.g., $(rv)$ for position and velocity, respectively), while subscripts indicate measurement times (e.g., $1{:}3$ for time indices $1$ through $3$ inclusive).}
\begin{equation}
\Matrix{H}^{(rv)}_1
=
\frac{\partial\<\Vector{y}_1}{\partial\SmallStateVector}
=
\bbm
1 & t_1 - \tau & -\bigl(v_0 + \breve{v}(t_1 - \tau)\bigr) \\
0 & 1 & -a(t_1 - \tau)
\ebm.
\end{equation}
Since $\Matrix{H}^{(rv)}_1\!\! \in \Real^{2 \times 3}$, it cannot have full column rank, and the parameters are not jointly identifiable.
To characterize the nullspace, consider a perturbation $\delta\SmallStateVector \Defined \smash{\bbm \delta r_0\! & \delta v_0\! & \delta\tau \ebm^{\T}}$ satisfying $\Matrix{H}^{(rv)}_1 \delta\SmallStateVector = \Vector{0}_{2 \times 1}$. Then
\begin{equation*}
\delta r_0 + (t_1 - \tau)\<\delta v_0
-
\bigl(v_0 + \breve{v}(t_1 - \tau)\bigr)\<\delta\tau
=
0,
\quad
\delta v_0 - a(t_1 - \tau)\<\delta\tau
=
0.
\end{equation*}
Hence,
\begin{equation*}
\delta v_0 = a(t_1 - \tau)\<\delta\tau,
\quad
\delta r_0
=
\bigl(v_0 + \breve{v}(t_1 - \tau) - 
(t_1 - \tau)\< a(t_1 - \tau)\bigr)\<\delta\tau,
\end{equation*}
so the nullspace is one-dimensional,
\begin{equation}
\label{eqn:SGal1_rv_nullspace}
\Ker{\Matrix{H}^{(rv)}_1}
=
\Span{
\bbm
v_0 + \breve{v}(t_1 - \tau) - (t_1 - \tau)\< a(t_1 - \tau) \\
a(t_1 - \tau) \\
1
\ebm
}.
\end{equation}
This direction corresponds to the symmetry described in \Cref{subsubsec:SGal1_one_full_measurement}: a change in delay can be compensated by a corresponding change in the initial condition, leaving the measurement unchanged to first order.

\subsubsection{Two Full-State Measurements}
\label{subsubsec:SGal1_two_full_measurements_jacobian}

For two full-state measurements at times $t_1$ and $t_2$, the stacked Jacobian is
\begin{equation*}
\Matrix{H}^{(rv)}_{1:2}
=
\bbm
\Matrix{H}^{(rv)}_1 \\[0.5mm]
\Matrix{H}^{(rv)}_2
\ebm,
\end{equation*}
which has full column rank whenever
\begin{equation*}
a(t_2 - \tau)\neq a(t_1 - \tau).
\end{equation*}
Therefore, two delayed full-state measurements are sufficient for local identifiability when the accelerations at the two delayed measurement times are unequal. 
This is the same conclusion reached by \Cref{eqn:SGal1_two_measurement_condition}.

\subsubsection{Two Position Measurements}
\label{subsubsec:SGal1_two_position_measurements_jacobian}

If only two position measurements are available, the stacked Jacobian is
\begin{equation}
\Matrix{H}^{(r)}_{1:2}
=
\bbm
1 & t_1 - \tau & -\bigl(v_0 + \breve{v}(t_1 - \tau)\bigr) \\[0.5mm]
1 & t_2 - \tau & -\bigl(v_0 + \breve{v}(t_2 - \tau)\bigr)
\ebm.
\end{equation}
This matrix cannot have full column rank; two delayed position measurements are therefore insufficient for local identifiability.
To characterize the nullspace, consider a perturbation $\delta\SmallStateVector$ satisfying $\Matrix{H}_{1:2}^{(r)}\,\delta\SmallStateVector = \Vector{0}_{2\times 1}$.
Subtracting the two rows yields
\begin{equation*}
(t_2 - t_1)\<\delta v_0
-
\bigl(\breve{v}(t_2 - \tau) - \breve{v}(t_1 - \tau)\bigr)\<\delta\tau
= 0,
\end{equation*}
and define
\begin{equation*}
\gamma
\Defined
\frac{\breve{v}(t_2 - \tau) - \breve{v}(t_1 - \tau)}{t_2 - t_1}.
\end{equation*}
Then
\begin{equation*}
\delta v_0 = \gamma\<\delta\tau,
\end{equation*}
and substitution into the first row gives
\begin{equation*}
\delta r_0
=
\bigl(v_0 + \breve{v}(t_1 - \tau) - (t_1 - \tau)\<\gamma\bigr)\<\delta\tau.
\end{equation*}
Hence, the nullspace is one-dimensional,
\begin{equation}
\Ker{\Matrix{H}^{(r)}_{1:2}}
=
\Span{
\bbm
v_0 + \breve{v}(t_1 - \tau) - (t_1 - \tau)\<\gamma \\
\gamma\\
1
\ebm
}.
\end{equation}
As in \Cref{eqn:SGal1_rv_nullspace}, this null direction is again an instance of the symmetry from \Cref{subsubsec:SGal1_two_position_measurements}.
Here, the required change in initial velocity is governed by the slope $\gamma$ rather than by the instantaneous acceleration at a single time. %
With only position measurements, the data constrain the average acceleration over the measurement interval, but do not uniquely determine when along the trajectory the measurements were taken.

\subsubsection{Three or More Position Measurements}
\label{subsubsec:SGal1_three_position_measurements_jacobian}

Consider three position measurements taken at times $t_1$, $t_2$, and $t_3$. The stacked Jacobian is
\begin{equation}
\Matrix{H}^{(r)}_{1:3}
=
\bbm
1 & t_1 - \tau & -\bigl(v_0 + \breve{v}(t_1 - \tau)\bigr) \\[0.5mm]
1 & t_2 - \tau & -\bigl(v_0 + \breve{v}(t_2 - \tau)\bigr) \\[0.5mm]
1 & t_3 - \tau & -\bigl(v_0 + \breve{v}(t_3 - \tau)\bigr)
\ebm.
\end{equation}
A sufficient condition for local identifiability is that the Jacobian has full rank.
Since $\Matrix{H}^{(r)}_{1:3} \in \Real^{3 \times 3}$, the rank may be checked by evaluating its determinant.
Subtracting the first row from the second and third rows gives
\begin{equation*}
\Determinant{\Matrix{H}^{(r)}_{1:3}}
=
\Determinant{
\bbm
1 & t_1 - \tau & -\bigl(v_0 + \breve{v}(t_1 - \tau)\bigr) \\[0.5mm]
0 & t_2 - t_1 & -\bigl(\breve{v}(t_2 - \tau) - \breve{v}(t_1 - \tau)\bigr) \\[0.5mm]
0 & t_3 - t_1 & -\bigl(\breve{v}(t_3 - \tau) - \breve{v}(t_1 - \tau)\bigr)
\ebm}.
\end{equation*}
Hence,
\begin{equation*}
\Determinant{\Matrix{H}_{1:3}^{(r)}}
=
(t_3 - t_1)\bigl(\breve{v}(t_2 - \tau) - \breve{v}(t_1 - \tau)\bigr)
-
(t_2 - t_1)\bigl(\breve{v}(t_3 - \tau) - \breve{v}(t_1 - \tau)\bigr).
\end{equation*}
Therefore, three delayed position measurements are locally sufficient for identifiability whenever
\begin{equation*}
(t_3 - t_1)\bigl(\breve{v}(t_2 - \tau) - \breve{v}(t_1 - \tau)\bigr)
\neq
(t_2 - t_1)\bigl(\breve{v}(t_3 - \tau) - \breve{v}(t_1 - \tau)\bigr).
\end{equation*}
Equivalently, local identifiability may fail when the average accelerations over the intervals $[t_1 - \tau,\, t_2 - \tau]$ and $[t_1 - \tau,\, t_3 - \tau]$ coincide.
Additional position measurements may be incorporated by stacking more rows, but three measurements are the minimum required for the position-only Jacobian to be full rank.

\subsection{Uninformative Trajectories}
\label{subsec:SGal1_bad_trajectories}

Identifiability of the delay (and the initial condition) depends not only on the number of measurements, but also on whether the input is sufficiently exciting.
There exist trajectories for which the delay and initial condition cannot be uniquely determined, regardless of the number or type of measurements.
We refer to these trajectories as \emph{uninformative}.%
\footnote{These trajectories are sometimes referred to as \emph{degenerate} in the literature. We use the term \emph{uninformative} to emphasize that the loss of identifiability arises from insufficient excitation of the input, rather than any defect in the trajectory itself.}

Our results below concern sufficiency and not necessity.
The stated conditions guarantee that the delay and the initial condition cannot be recovered.
We assume that the input is constant from the initial time $t = 0$ through the latest delayed measurement time.
In fact, the input need only be constant over an interval containing $t_k - \tau - \alpha$ for every delayed measurement time, where $\alpha$ is a small positive or negative shift of the delay.
Motion prior to the interval enters every delayed state through a single fixed factor that is absorbed into the initial condition.
We make the stronger assumption here because it avoids carrying that factor through the derivation.

Constant-acceleration motion provides a standard example of an uninformative trajectory, and is in fact the only case for which loss of identifiability holds for all measurement times.
Suppose $a(t) = a_0$ for all $t$, so that
\begin{equation}
\breve{v}(t) = a_0\<t,
\quad
\breve{r}(t) = \frac{1}{2}a_0\<t^2.
\end{equation}
In this case, the state-transition matrix $\Matrix{\Phi}(t)$ is generated by the exponential map of a fixed Lie algebra element.
Because $\Matrix{\Phi}(0)=\Identity_{3}$ and $\breve{v}(0)=\breve{r}(0)=0$, the generator is not arbitrary; it must match the zero initial conditions built into $\Matrix{\Phi}(t)$.
Specifically,
\begin{equation}
\Matrix{\Phi}(t)
=
\Matexp{t\<\Vector{\xi}^{\wedge}_{c}},
\quad
\Vector{\xi}^{\wedge}_{c}
=
\bbm
0 & a_0 & 0 \\
0 & 0   & 1 \\
0 & 0   & 0
\ebm.
\end{equation}
More generally, replacing $\Vector{\xi}^{\wedge}_{c}$ by a different constant Lie algebra element would only add a constant-velocity contribution that can be absorbed into the initial condition.
The restriction comes from the choice that $\Matrix{\Phi}(0) = \Identity_3$ and does not affect the symmetry argument below.
It follows that
\begin{equation}
\Matrix{\Phi}(t_k - \tau)\<
\Matrix{\Phi}(t_k - \alttau)^{-1}
=
\Matexp{(\alttau - \tau)\<
\Vector{\xi}^{\wedge}_{c}},
\end{equation}
which is independent of $t_k$.

Returning to the measurement model, the same transformation used in the single-measurement analysis therefore applies simultaneously to all measurements.
Define
\begin{equation*}
\bar{\Matrix{X}}_0
=
\Matrix{T}(\tau - \alttau)\<
\Matrix{X}_0\<
\Matrix{\Phi}(\alttau - \tau),
\end{equation*}
so that, for all measurement times $t_k$,
\begin{equation*}
\Matrix{Y}_k
=
\Matrix{T}(\tau)\<
\Matrix{X}_0\<
\Matrix{\Phi}(t_k - \tau)
=
\Matrix{T}(\alttau)\<
\bar{\Matrix{X}}_0\<
\Matrix{\Phi}(t_k - \alttau).
\end{equation*}
The entire measurement set can therefore be explained by a different value of $\tau$ together with a corresponding change in the initial condition.
To connect this with \Cref{subsec:symmetries}, recall that the unknown parameter vector is
\begin{equation*}
\SmallStateVector
\Defined
\bbm
r_0\! & v_0\! & \tau
\ebm^{\T}.
\end{equation*}
Assume that $\tau$ lies in the interior of the admissible delay interval, and choose $\epsilon>0$ sufficiently small that $\tau + \alpha$ and all shifted delayed evaluation times remain admissible for every $\alpha \in (-\epsilon, \epsilon)$.
Define a transformation $\mathcal{S}_\alpha$ on these parameters by\footnote{The Lie bracket relation $[\Matrix{B}, \Matrix{K}] = \Matrix{P}$ gives the structure of this symmetry: shifting the delay changes not only the velocity component but also the position component.}
\begin{equation}
\mathcal{S}_\alpha :
(r_0, v_0, \tau)
\mapsto
\left(
r_0 + v_0\<\alpha + \frac{1}{2}a_0\<\alpha^2,
v_0 + a_0\<\alpha,
\tau + \alpha
\right).
\end{equation}
Under constant-acceleration motion, the parameter tuples $\SmallStateVector$ and $\mathcal{S}_\alpha(\SmallStateVector)$ generate the same measurement history.
This family therefore satisfies the conditions of \Cref{cor:symmetry_unidentifiable_discrete} (and also \Cref{thm:symmetry_unidentifiable_continuous}), so the parameters are not locally identifiable.

The symmetry is also reflected in the Jacobian-based analysis as a loss of rank.
The rank conditions for both full-state and position-only measurements fail under constant acceleration, since $a(t)$ is identical at all delayed measurement times and $\breve{v}(t)$ is affine.
More generally, identifiability fails whenever the trajectory is locally indistinguishable from constant-acceleration motion at the delayed measurement times.

We revisit this example in \Cref{subsec:SGal3_bad_trajectories} when we extend the analysis of uninformative trajectories to $\LieGroupSGal{3}$.
The key point is that symmetry is exactly what leads to unidentifiability: when a shift in the delay can be absorbed into the initial condition, the measurements no longer uniquely determine the parameters.

\begin{infobox}[t!]
\begin{mdframed}
\textbf{What's Wrong With Recursive Filtering?}
\vspace{0.5\baselineskip}
\par\noindent To come full circle, we briefly return to the question of what unidentifiability implies for estimation.
A recursive estimator such as the EKF maintains uncertainty over the state and over any additional parameters, treating them as unknown.
Each measurement update adds the term $\Matrix{H}^\T\Matrix{R}^{-1}\Matrix{H}$ to the information matrix, which shrinks the filter covariance.
The Jacobian $\Matrix{H}$ must be evaluated somewhere, and the filter has no choice but to use the current estimate $\hat{\SmallStateVector}$ in place of the (unavailable) true value $\SmallStateVector^\star$.
When the system is unidentifiable, this is problematic.
For the $\LieGroupSGal{1}$ example, the information gained along the unobservable direction $\Vector{n}$, the nullspace vector of \Cref{eqn:SGal1_rv_nullspace}, is the scalar
\begin{equation}
\label{eqn:SGal1_info_gain}
\Vector{n}^\T
\Matrix{H}^{(rv)\T}_1
\Matrix{R}^{-1}
\Matrix{H}^{(rv)}_1
\Vector{n}
=
\bigl\lVert
\Matrix{R}^{-1/2}
\Matrix{H}^{(rv)}_1
\Vector{n}\bigr\rVert^2.
\end{equation}
At the true delay, $\Matrix{H}^{(rv)}_1\Vector{n} = \Vector{0}$ by construction, so the update gains no information along $\Vector{n}$.
Suppose instead the filter carries a delay estimate $\hat{\tau} \neq \tau$ while the kinematic estimates are exact.
Forming the Jacobian at $\hat{\tau}$ and cancelling the common terms gives
\begin{equation}
\label{eqn:SGal1_spurious_residual}
\Matrix{H}^{(rv)}_1(\hat{\tau})\<\Vector{n}
=
\bbm
\breve{v}(t_1 - \tau) - \breve{v}(t_1-\hat{\tau}) + a(t_1-\tau)\<
(\tau-\hat{\tau}) \\[2pt]
a(t_1 - \tau) - a(t_1 - \hat{\tau})
\ebm,
\end{equation}
which is generically nonzero for a time-varying input.
Taking $\hat{\tau} = 0$, for example, the velocity entry becomes $a(t_1 - \tau) - a(t_1)$.
Since the residual in \Cref{eqn:SGal1_spurious_residual} is nonzero, the information gain \Cref{eqn:SGal1_info_gain} is strictly positive.
The filter generically shrinks the covariance along $\Vector{n}$, a direction the measurement carries no information about.
Because the delay is constant (with no process noise added, typically), each update generically reduces the covariance along $\Vector{n}$, with no mechanism to undo it, and the estimator can become inconsistent.
This issue and its consequences are discussed by Huang in~\cite{2017_Huang_Towards}, for example.
\end{mdframed}
\vspace{-\baselineskip}
\end{infobox}

\section{The Galilean Group $\LieGroupSGal{3}$}
\label{sec:SGal3_group}

The special Galilean group $\LieGroupSGal{3}$ is a ten-dimensional Lie group.
We include just enough background on $\LieGroupSGal{3}$ for our purposes; more details are available in~\cite{2023_Kelly_Galilean} and~\cite{2025_Mahony_Galilean}.
The group describes transformations between inertial frames undergoing constant relative motion, and combines rotations, Galilean boosts, and spacetime translations.
Elements of $\LieGroupSGal{3}$ can be written as $5 \times 5$ matrices,%
\begin{equation}
\label{eqn:SGal3_definition}
\LieGroupSGal{3} \Defined
\begin{Bmatrix}
\Matrix{X} =
\bbm 
\Matrix{C} & \Vector{v} & \Vector{r} \\ 
\Zero & 1 & \eta \\
\Zero & 0 & 1
\ebm
\in \Real^{5 \times 5}
\ \vast\vert \ 
\Matrix{C} \in \LieGroupSO{3}, \,
\Vector{v} \in \Real^3,\,
\Vector{r} \in \Real^3,\,
\eta \in \Real
\end{Bmatrix}.
\end{equation}
Here, $\Matrix{C}$ is a rotation matrix, $\Vector{v} \in \Real^3$ is the Galilean boost, $\Vector{r} \in \Real^3$ is the spatial translation, and $\eta \in \Real$ is the time translation.
The group composition operation is matrix multiplication.
The inverse of $\Matrix{X}$ is
\begin{equation}
\label{eqn:SGal3_inverse}
\Inv{\Matrix{X}} =
\bbm
 \Transpose{\Matrix{C}} &
-\Transpose{\Matrix{C}}\Vector{v} &
-\Transpose{\Matrix{C}}\left(\Vector{r} - \Vector{v}\<\eta\right) \\ 
 \Zero & 1 & -\eta \\
 \Zero & 0 & 1
\ebm,
\end{equation}
such that $\Matrix{X}\Inv{\Matrix{X}} = \Identity_{5}$.

A group element may be interpreted as a transformation between inertial frames in constant relative motion, and acts on spacetime coordinates $(\Vector{p}, t)$, with $\Vector{p} \in \Real^3$, according to
\begin{equation*}
(\Vector{p}, t) \mapsto (\Matrix{C}\<\Vector{p} + \Vector{v}\<t + \Vector{r},\, t + \eta).
\end{equation*}
Using homogeneous coordinates $\bbm \Vector{p}^{\T}\! & t\! & 1 \ebm^\T$, this action can be written as
\begin{equation*}
\bbm
\Vector{p} \\ t \\ 1
\ebm
\mapsto
\Matrix{X}
\bbm
\Vector{p} \\ t \\ 1
\ebm
=
\bbm
\Matrix{C} & \Vector{v} & \Vector{r} \\
\Zero & 1 & \eta \\
\Zero & 0 & 1
\ebm
\bbm
\Vector{p} \\ t \\ 1
\ebm.
\end{equation*}
For two group elements
\begin{equation*}
\Matrix{X}_1 =
\bbm
\Matrix{C}_1 & \Vector{v}_1 & \Vector{r}_1 \\
\Zero & 1 & \eta_1 \\
\Zero & 0 & 1
\ebm,
\quad
\Matrix{X}_2 =
\bbm
\Matrix{C}_2 & \Vector{v}_2 & \Vector{r}_2 \\
\Zero & 1 & \eta_2 \\
\Zero & 0 & 1
\ebm,
\end{equation*}
their composition is
\begin{equation*}
\Matrix{X}_1\Matrix{X}_2
=
\bbm
\Matrix{C}_1\Matrix{C}_2 &
\Matrix{C}_1\Vector{v}_2 + \Vector{v}_1 &
\Matrix{C}_1\Vector{r}_2 + \Vector{v}_1\eta_2 + \Vector{r}_1 \\
\Zero & 1 & \eta_1 + \eta_2 \\
\Zero & 0 & 1
\ebm.
\end{equation*}
The term $\Vector{v}_1\eta_2$ reflects the coupling between boosts and time translations, while the terms $\Matrix{C}_1\Vector{v}_2$ and $\Matrix{C}_1\Vector{r}_2$ reflect the action of rotations on boosts and spatial translations.

\subsection{The Lie Algebra $\LieAlgebraSGal{3}$}

The Lie algebra $\LieAlgebraSGal{3}$ consists of matrices of the form
\begin{equation}
\label{eqn:sgal3_algebra_definition}
\LieAlgebraSGal{3}
\Defined
\begin{Bmatrix}
\Matrix{\Xi} =
\bbm
\Vector{\phi}^{\wedge} & \Vector{\nu} & \Vector{\rho} \\
\Zero & 0 & \iota \\
\Zero & 0 & 0
\ebm
\in \Real^{5 \times 5}
\ \vast\vert \
\Vector{\phi},\, \Vector{\nu},\, \Vector{\rho} \in \Real^3,\,
\iota \in \Real
\end{Bmatrix},
\end{equation}
where $\Vector{\phi} \in \Real^3$ parameterizes the rotational part, $\Vector{\nu} \in \Real^3$ the boost part, $\Vector{\rho} \in \Real^3$ the translational part, and $\iota \in \Real$ the temporal part.
We overload the wedge operator as a mapping $\Real^{10} \rightarrow \LieAlgebraSGal{3}$,
\begin{equation}
\label{eqn:SGal3_wedge}
\Vector{\xi}^{\wedge}
=
\bbm
\Vector{\rho} \\
\Vector{\nu} \\
\Vector{\phi} \\
\iota
\ebm^{\wedge}
=
\bbm
\Vector{\phi}^{\wedge} & \Vector{\nu} & \Vector{\rho} \\
\Zero & 0 & \iota \\
\Zero & 0 & 0
\ebm,
\end{equation}
where $\Vector{\xi} = \bbm \Vector{\rho}^{\T}\;\; \Vector{\nu}^{\T}\;\; \Vector{\phi}^{\T}\;\; \iota\ebm^{\T} \in \Real^{10}$.
The corresponding vee operator is defined by
\begin{equation*}
\Vector{\xi}^{\wedge} = \Matrix{\Xi}
\quad \Longleftrightarrow \quad
\Matrix{\Xi}^{\vee} = \Vector{\xi}.
\end{equation*}

\subsection{The Exponential Map}
\label{subsec:SGal3_exp_log_maps}

Let $\Vector{\xi} = \bbm \Vector{\rho}^{\T}\;\; \Vector{\nu}^{\T}\;\; \Vector{\phi}^{\T}\;\; \iota \ebm^{\T} \in \Real^{10}$.
The exponential map is defined by the matrix exponential,
\begin{equation*}
\Matexp{\Vector{\xi}^{\wedge}}
=
\sum_{n=0}^{\infty}
\frac{1}{n!}
\left(\Vector{\xi}^{\wedge}\right)^n.
\end{equation*}
In closed form for $\LieGroupSGal{3}$,
\begin{equation}
\label{eqn:SGal3_exp_closed}
\Matexp{\Vector{\xi}^{\wedge}}
=
\bbm
\Matrix{C}(\Vector{\phi}) &
\Matrix{D}(\Vector{\phi})\<\Vector{\nu} &
\Matrix{D}(\Vector{\phi})\<\Vector{\rho} + 
\Matrix{E}(\Vector{\phi})\<\Vector{\nu}\iota \\
\Zero & 1 & \iota \\
\Zero & 0 & 1
\ebm,
\end{equation}
where
\begin{align}
\label{eqn:SGal3_C_closed}
\Matrix{C}(\Vector{\phi})
&=
\Identity_{3}
+
\sin(\phi)\<\uw
+
\bigl(1-\cos(\phi)\bigr)\uw\uw, \\[2mm]
\label{eqn:SGal3_D_closed}
\Matrix{D}(\Vector{\phi})
&=
\Identity_{3}
+
\left(\frac{1-\cos(\phi)}{\phi}\right)\uw
+
\left(\frac{\phi-\sin(\phi)}{\phi}\right)\uw\uw, \\
\label{eqn:SGal3_E_closed}
\Matrix{E}(\Vector{\phi})
&=
\frac{1}{2}\Identity_{3}
+
\left(\frac{\phi-\sin(\phi)}{\phi^{2}}\right)\uw
+
\left(\frac{\phi^{2} + 2\cos(\phi)-2}{2\phi^{2}}\right)\uw\uw,
\end{align}
and $\Vector{\phi} = \phi\<\Vector{u}$, where $\phi = \Norm{\Vector{\phi}}$ and $\Vector{u}$ is a unit vector.

\section{Aided Navigation and Identifiability}
\label{sec:aided_identifiability}

We now consider aided inertial navigation in three-dimensional space.
The process model is driven by known IMU angular-rate and specific-force measurements, both resolved in the vehicle frame.
In addition to the unknown initial condition and delay, the model includes gyroscope and accelerometer biases.
For the delay-estimation analysis, we assume these biases are constant over the relevant time window.
Following our development for $\LieGroupSGal{1}$, we initially neglect gravity to keep the essential structure clear.%
\footnote{As we will see, the identifiability results carry over unchanged to the model that includes gravity.}
We then incorporate gravity, and show that its effect on the system's time evolution still fits naturally within the same symmetry framework.%
\footnote{The Earth's rotation is neglected, since its effect is negligible over the short time windows considered here.}

Relative to the one-dimensional case, the analysis is more involved for two reasons.
First, the state and parameter spaces are higher dimensional, and the orientation must also be estimated, introducing a constrained nonlinear state component.
Second, the unknown bias terms enter the kinematics and therefore become part of the identifiability problem.
Accordingly, we ask how many measurements, and of what type, are sufficient to determine the unknown initial condition, the delay, and the biases.

In this section, we use $\StateVector$ for the parameter tuple, including any group-valued elements, and $\SmallStateVector$ for its local-coordinate representation.
This distinction is important for the orientation component of the initial condition.
We follow the notation from \Cref{sec:back_to_basics} wherever possible, since most of the same ideas and conclusions carry over.
Any changes specific to the three-dimensional setting are introduced where needed.

\subsection{Aided Inertial Navigation on $\LieGroupSGal{3}$}
\label{subsec:SGal3_system}

The aided INS is driven by IMU measurements.
The state is represented on $\LieGroupSGal{3}$ as
\begin{equation}
\label{eqn:SGal3_state}
\Matrix{X}(t)
=
\bbm
\Matrix{C}(t) & \Vector{v}(t) & \Vector{r}(t) \\
\Zero & 1 & t \\
\Zero & 0 & 1
\ebm,
\end{equation}
where $\Matrix{C}(t) \in \LieGroupSO{3}$ is the orientation of the vehicle frame with respect to the fixed inertial frame, and $\Vector{v}(t), \Vector{r}(t) \in \Real^3$ are the velocity and position, respectively, expressed in the inertial frame.
The IMU measurements $\Vector{\omega}_m(t)$ and $\Vector{a}_m(t)$ are given in the vehicle frame.

The unknown parameters are
\begin{equation*}
\StateVector
\Defined
\left(
\Matrix{X}_0,\,
\tau,\,
\Vector{b}_{\omega},\,
\Vector{b}_a
\right),
\end{equation*}
where $\Matrix{X}_0 \in \LieGroupSGal{3}$ is the initial condition, $\tau \in \Real$ is the delay, and $\Vector{b}_{\omega},\Vector{b}_a \in \Real^3$ are the gyroscope and accelerometer biases, respectively.
Without loss of generality, we take the initial time to be zero, placing $\Matrix{X}_0$ in the isochronous subgroup of $\LieGroupSGal{3}$, and write
\begin{equation}
\label{eqn:SGal3_initial_state}
\Matrix{X}_0
=
\bbm
\Matrix{C}_0 & \Vector{v}_0 & \Vector{r}_0 \\
\Zero & 1 & 0 \\
\Zero & 0 & 1
\ebm
\in \LieGroupSGal{3},
\end{equation}
where $\Matrix{C}_0 = \Matrix{C}(0)$, $\Vector{v}_0 = \Vector{v}(0)$, and $\Vector{r}_0 = \Vector{r}(0)$ are unknown.

Following the presentation in \Cref{subsec:SGal1_system}, we develop the analysis initially in the \emph{gravity-free} setting; the gravity vector is added to the model in \Cref{subsec:SGal3_now_with_gravity}.
The continuous-time kinematics are
\begin{align}
\label{eqn:SGal3_kinematics_C}
\dot{\Matrix{C}}(t)
&=
\Matrix{C}(t)\bigl(\Vector{\omega}_m(t) - \Vector{b}_{\omega}\bigr)^{\wedge}, \\
\label{eqn:gravity_free_force}
\dot{\Vector{v}}(t)
&=
\Matrix{C}(t)\bigl(\Vector{a}_m(t) - \Vector{b}_a\bigr), \\
\label{eqn:SGal3_kinematics_r}
\dot{\Vector{r}}(t)
&=
\Vector{v}(t),
\end{align}
where $\dot{t} = 1$ by definition, and where $\Vector{\omega}_m(t),\, \Vector{a}_m(t) \in \Real^3$ are the measured vehicle-frame angular rates and specific forces, respectively.
We define the bias-corrected angular-rate and specific-force vectors,
\begin{equation}
\label{eqn:SGal3_bias_corrected_inputs}
\Vector{\omega}(t) \Defined \Vector{\omega}_m(t) - \Vector{b}_{\omega},
\quad
\Vector{s}(t) \Defined \Vector{a}_m(t) - \Vector{b}_a,
\end{equation}
so that \Cref{eqn:SGal3_kinematics_C,eqn:gravity_free_force} are $\dot{\Matrix{C}}(t) = \Matrix{C}(t)\<\Vector{\omega}(t)^{\wedge}$ and $\dot{\Vector{v}}(t) = \Matrix{C}(t)\<\Vector{s}(t)$.
Both $\Vector{\omega}(\cdot)$ and $\Vector{s}(\cdot)$ are then available directly from the inertial measurements, assuming the biases are known.

It is worth clarifying the form of \Cref{eqn:gravity_free_force}.
An accelerometer measures \emph{specific force}, the non-gravitational force per unit mass resolved in the vehicle frame, and not the \emph{coordinate acceleration} $\Vector{a}(t) \Defined \dot{\Vector{v}}(t) = \ddot{\Vector{r}}(t)$ of the vehicle with respect to the inertial frame.
In the presence of a uniform gravitational field, the two are related by $\Vector{a}(t) = \Matrix{C}(t)\<\Vector{s}(t) + \Vector{g}$, and an additive $\Vector{g}$ term would appear in \Cref{eqn:gravity_free_force}.
With $\Vector{g} = \Zero$, the measured specific force is just the coordinate acceleration resolved in the vehicle frame, $\Vector{s}(t) = \Matrix{C}(t)^\T\Vector{a}(t)$, and $\Vector{s}(\cdot)$ can be treated as a known model input.

It is convenient to express the trajectory on the group as
\begin{equation*}
\Matrix{X}(t) = \Matrix{X}_0\<\Matrix{\Phi}(t),
\end{equation*}
where $\Matrix{\Phi}(t) \in \LieGroupSGal{3}$ is the state-transition matrix,
\begin{equation}
\label{eqn:SGal3_phi_matrix}
\Matrix{\Phi}(t)
=
\bbm
\breve{\Matrix{C}}(t) & \breve{\Vector{v}}(t) & \breve{\Vector{r}}(t) \\
\Zero & 1 & t \\
\Zero & 0 & 1
\ebm,
\end{equation}
with $\breve{\Matrix{C}}(0)=\Identity_3$, $\breve{\Vector{v}}(0)=\Zero$, and $\breve{\Vector{r}}(0)=\Zero$.
The components satisfy
\begin{align}
\label{eqn:SGal3_C_breve_dot}
\dot{\breve{\Matrix{C}}}(t) 
& = 
\breve{\Matrix{C}}(t)\<\Vector{\omega}(t)^{\wedge}, \\
\label{eqn:SGal3_v_breve_dot}
\dot{\breve{\Vector{v}}}(t)
& = \breve{\Matrix{C}}(t)\<\Vector{s}(t), \\
\label{eqn:SGal3_r_breve_dot}
\dot{\breve{\Vector{r}}}(t) 
& = 
\breve{\Vector{v}}(t).
\end{align}

We now develop the measurement model, following \Cref{subsec:SGal1_system}.
The structure carries over directly, with the scalar kinematics replaced by their $\LieGroupSGal{3}$ counterparts.
The Lie algebra is now ten-dimensional, so the time translation direction is a vector in $\Real^{10}$ rather than $\Real^{3}$,
\begin{equation}
\label{eqn:sgal3_time_translation_generator}
\Vector{\xi}_d
\Defined
\bbm
\Zero_{1 \times 9} & 1
\ebm^\T
\in \Real^{10}.
\end{equation}
Let
\begin{equation}
\label{eqn:SGal3_pure_time_translation}
\Matrix{T}(\beta)
\Defined
\Matexp{\beta\<\Vector{\xi}_d^{\wedge}}
\end{equation}
denote the corresponding time translation, where the wedge maps $\Vector{\xi}_d$ to the $5 \times 5$ matrix representation.
Under left multiplication, $\Matrix{T}(\beta)$ shifts the group time coordinate by $\beta$ while leaving the orientation, velocity, and position components unchanged.
This is the operation used in the measurement construction below.

As in the one-dimensional case, the measured kinematics are the values at the delayed time $t_k - \tau$, but the measurement arrives at time $t_k$.
A left time translation $\Matrix{T}(\tau)$ advances the group time coordinate from $t_k - \tau$ to $t_k$.
Define
\begin{equation}
\label{eqn:SGal3_delayed_measurement}
\Matrix{Y}_k
\Defined
\Matrix{T}(\tau)\<
\Matrix{X}(t_k - \tau)
=
\Matrix{T}(\tau)\<
\Matrix{X}_0\<\Matrix{\Phi}(t_k - \tau)
=
\bbm
\Matrix{C}(t_k - \tau) & \Vector{v}(t_k - \tau) & \Vector{r}(t_k - \tau) \\
\Zero & 1 & t_k \\
\Zero & 0 & 1
\ebm,
\quad
k = 1, \dots, n.
\end{equation}
The initial condition $\Matrix{X}_0$ retains a zero time coordinate and remains in the isochronous subgroup.
Using the propagation equations above, the delayed orientation, velocity, and position are
\begin{align}
\Matrix{C}_k
&=
\Matrix{C}(t_k - \tau)
=
\Matrix{C}_0\<
\breve{\Matrix{C}}(t_k - \tau), \\
\Vector{v}_k
&=
\Vector{v}(t_k - \tau)
=
\Matrix{C}_0\<
\breve{\Vector{v}}(t_k - \tau) + \Vector{v}_0, \\
\Vector{r}_k
&=
\Vector{r}(t_k - \tau)
=
\Matrix{C}_0\<
\breve{\Vector{r}}(t_k - \tau)
+
(t_k - \tau)\<
\Vector{v}_0 
+
\Vector{r}_0.
\end{align}
The unknown parameters are the initial condition $\Matrix{X}_0$, whose components are $(\Matrix{C}_0, \Vector{v}_0, \Vector{r}_0)$ as given above, together with the delay $\tau$ and the constant biases $(\Vector{b}_{\omega}, \Vector{b}_a)$, while the functions $\breve{\Matrix{C}}(\cdot)$, $\breve{\Vector{v}}(\cdot)$, and $\breve{\Vector{r}}(\cdot)$ are determined by the measured IMU history together with the biases.

\subsection{Identifiability}
\label{subsec:SGal3_identifiability}

We now examine the identifiability of delayed aided navigation on $\LieGroupSGal{3}$.
As in the one-dimensional case, we start by studying the constraints induced by the measurements directly.
We consider three measurement types: full-state measurements (position, velocity, and orientation), pose measurements (position and orientation), and position-only measurements.

\subsubsection{A Single Full-State Measurement}
\label{subsubsec:SGal3_one_measurement}

With a single delayed full-state measurement at time $t_1$,
\begin{equation*}
\Matrix{Y}_1
=
\Matrix{T}(\tau)\<
\Matrix{X}(t_1 - \tau)
=
\Matrix{T}(\tau)\<
\Matrix{X}_0\<
\Matrix{\Phi}(t_1 - \tau),
\end{equation*}
the unknown initial condition again cannot be eliminated. As in the $\LieGroupSGal{1}$ case, for any admissible alternative delay $\alttau$ define
\begin{equation*}
\bar{\Matrix{X}}_0
=
\Matrix{T}(\tau - \alttau)\<
\Matrix{X}_0\<
\Matrix{\Phi}(t_1 - \tau)\<
\Matrix{\Phi}(t_1 - \alttau)^{-1},
\end{equation*}
such that
\begin{equation*}
\Matrix{Y}_1
=
\Matrix{T}(\alttau)\<
\bar{\Matrix{X}}_0\<
\Matrix{\Phi}(t_1 - \alttau).
\end{equation*}
The same observation can be explained by different values of $\tau$ together with corresponding changes in the initial condition.
The delay and initial condition are therefore jointly unidentifiable from a single full-state measurement.
As before, the transformation of $(\Matrix{X}_0, \tau)$ defines a continuous symmetry satisfying \Cref{cor:symmetry_unidentifiable_discrete}.
The biases enter only through the integrated quantities $\smash{\breve{\Matrix{C}}}$, $\smash{\breve{\Vector{v}}}$, and $\smash{\breve{\Vector{r}}}$, which a single measurement cannot separate from the initial condition, so they are also unidentifiable.

\subsubsection{Three Pose Measurements}
\label{subsubsec:SGal3_three_pose_measurements}

While two full-state measurements are generically sufficient, that case is more tedious to work through algebraically.
We instead move directly to the case of three delayed pose measurements (i.e., position and orientation), which is more useful in practice and easier to treat.
Let
\begin{equation*}
\Matrix{C}_k = \Matrix{C}(t_k - \tau),
\quad
\Vector{r}_k = \Vector{r}(t_k - \tau),
\quad
k = 1, 2, 3.
\end{equation*}
Using the propagation equations above, we establish identifiability by analyzing the measurement constraints in sequence.

\paragraph{Step 1: Recover $(\tau, \Vector{b}_{\omega})$ from Relative Orientations.}
From the orientation equation,
\begin{equation*}
\Matrix{C}_k
=
\Matrix{C}_0\<\breve{\Matrix{C}}(t_k - \tau),
\end{equation*}
so for any pair $(i,j)$,
\begin{equation*}
\Transpose{\Matrix{C}}_i\<\Matrix{C}_j
=
\Transpose{\breve{\Matrix{C}}(t_i - \tau)}
\breve{\Matrix{C}}(t_j - \tau),
\end{equation*}
and the initial orientation $\Matrix{C}_0$ is eliminated.
Using the pairs $(1, 2)$ and $(2, 3)$ gives two independent relative orientation constraints, giving six scalar equations in the four unknowns $(\tau, \Vector{b}_{\omega})$.
Under sufficiently informative motion, these determine $\tau$ and $\Vector{b}_{\omega}$ locally.

\paragraph{Step 2: Recover the Initial Orientation.}
Once $\tau$ and $\Vector{b}_{\omega}$ are known, the propagated rotation $\breve{\Matrix{C}}(t_k - \tau)$ is known, and the initial orientation is recovered from any single measurement:
\begin{equation*}
\Matrix{C}_0
=
\Matrix{C}_1\breve{\Matrix{C}}(t_1 - \tau)^{-1}.
\end{equation*}

\paragraph{Step 3: Recover $(\Vector{v}_0, \Vector{b}_a)$ from Position Differences.}
From the position equation,
\begin{equation*}
\Vector{r}_k
=
\Matrix{C}_0\<\breve{\Vector{r}}(t_k - \tau)
+
(t_k - \tau)\<\Vector{v}_0
+
\Vector{r}_0.
\end{equation*}
Subtracting two such equations eliminates $\Vector{r}_0$:
\begin{equation}
\label{eqn:SGal3_pose_difference}
\Vector{r}_j - \Vector{r}_i
=
\Matrix{C}_0
\bigl(
\breve{\Vector{r}}(t_j - \tau) - \breve{\Vector{r}}(t_i - \tau)
\bigr)
+
(t_j - t_i)\<\Vector{v}_0.
\end{equation}
Once $\tau$, $\Vector{b}_{\omega}$, and $\Matrix{C}_0$ are known, the remaining unknowns in \Cref{eqn:SGal3_pose_difference} are $\Vector{v}_0$ and $\Vector{b}_a$, since the propagated quantity $\breve{\Vector{r}}(\cdot)$ depends on $\Vector{b}_a$.
Using the pairs $(1, 2)$ and $(2, 3)$ gives six scalar equations in the six unknowns $(\Vector{v}_0, \Vector{b}_a)$.
Under generic motion, these determine $\Vector{v}_0$ and $\Vector{b}_a$ locally.

\paragraph{Step 4: Recover the Initial Position.}
Once $\tau$, $\Vector{b}_{\omega}$, $\Matrix{C}_0$, $\Vector{v}_0$, and $\Vector{b}_a$ are known, the initial position follows from any one position measurement:
\begin{equation*}
\Vector{r}_0
=
\Vector{r}_1
-
\Matrix{C}_0\<\breve{\Vector{r}}(t_1 - \tau)
-
(t_1 - \tau)\<\Vector{v}_0.
\end{equation*}
Three pose measurements are therefore generically sufficient for local identifiability.

\subsubsection{Six or More Position Measurements}
\label{subsubsec:SGal3_six_position_measurements}

Finally, we consider the case in which only delayed position measurements are available:
\begin{equation*}
\Vector{r}_k = \Vector{r}(t_k-\tau),
\quad
k = 1, \dots, n.
\end{equation*}
From the position model,
\begin{equation*}
\Vector{r}_k
=
\Matrix{C}_0\<\breve{\Vector{r}}(t_k - \tau)
+
(t_k - \tau)\<\Vector{v}_0
+
\Vector{r}_0,
\end{equation*}
differencing eliminates $\Vector{r}_0$:
\begin{equation*}
\Vector{r}_j - \Vector{r}_i
=
\Matrix{C}_0
\bigl(
\breve{\Vector{r}}(t_j - \tau)-\breve{\Vector{r}}(t_i - \tau)
\bigr)
+
(t_j - t_i)\<\Vector{v}_0.
\end{equation*}
After eliminating $\Vector{r}_0$, the remaining unknowns are $(\Matrix{C}_0, \Vector{v}_0, \tau, \Vector{b}_{\omega}, \Vector{b}_a)$, which have total dimension $13$.
Each independent position difference contributes three scalar constraints.
Hence, with $n$ position measurements, at most $3\<(n - 1)$ independent scalar constraints can be formed.
A necessary condition for local identifiability is therefore $3\<(n - 1) \ge 13$, which implies $n \ge 6$.
In generic terms, six position measurements are the first case in which local identifiability is possible.
We do not work through the full elimination argument here.

\subsection{Measurement Jacobians and Nullspace Structure}
\label{subsec:SGal3_jacobian_nullspace}

We now revisit the $\LieGroupSGal{3}$ identifiability analysis using the Jacobian of the measurement function with respect to the unknown parameters.
As in the $\LieGroupSGal{1}$ case, the goal is to characterize local identifiability through the rank and nullspace of the Jacobian.
Using local coordinates for $\LieGroupSO{3}$, we parameterize perturbations by
\begin{equation}
\delta\SmallStateVector
\Defined
\bbm
\delta\Vector{r}_0^\T &
\delta\Vector{v}_0^\T &
\delta\boldsymbol{\phi}_0^\T &
\delta\tau &
\delta\Vector{b}_{\omega}^\T &
\delta\Vector{b}_a^\T
\ebm^\T
\in \Real^{16},
\end{equation}
where $\delta\boldsymbol{\phi}_0 \in \Real^3$ is the local perturbation of the initial orientation.
We use a right-multiplicative orientation perturbation,
\begin{equation*}
\bar{\Matrix{C}}_0
=
\Matrix{C}_0\<\Matexp{\delta\boldsymbol{\phi}_0^\wedge},
\end{equation*}
and define the nominal and perturbed delayed evaluation times as
\begin{equation*}
t_{d,k} \Defined t_k - \tau,
\qquad
\bar{t}_{d,k} \Defined t_k - \alttau.
\end{equation*}
The induced perturbation in the measured orientation is
\begin{equation*}
\delta\boldsymbol{\vartheta}_k
=
\Matlog{
\Matrix{C}(t_{d,k})^\T
\bar{\Matrix{C}}(\bar{t}_{d,k})
}^\vee,
\end{equation*}
where $\bar{\Matrix{C}}(t)$ is propagated from $\bar{\Matrix{C}}_0$ under the perturbed gyroscope bias.
Making use of the propagated quantities introduced earlier:
\begin{align*}
\Matrix{C}(t)
&=
\Matrix{C}_0\<\breve{\Matrix{C}}(t), \\
\Vector{v}(t)
&=
\Matrix{C}_0\<\breve{\Vector{v}}(t) + \Vector{v}_0, \\
\Vector{r}(t)
&=
\Matrix{C}_0\<\breve{\Vector{r}}(t)+\Vector{v}_0\<t + \Vector{r}_0,
\end{align*}
where $\breve{\Matrix{C}}(t)$, $\breve{\Vector{v}}(t)$, and $\breve{\Vector{r}}(t)$ are determined by the bias-corrected IMU history.
The full-state measurement returns the position, velocity, and orientation components of $\Matrix{Y}_k$, extracted by the map $\pi : \LieGroupSGal{3} \to \Real^3 \times \Real^3 \times \LieGroupSO{3}$,
\begin{equation*}
\pi(\Matrix{Y}_k)
\Defined
\bigl(
\Vector{r}(t_{d,k}),\,
\Vector{v}(t_{d,k}),\,
\Matrix{C}(t_{d,k})
\bigr).
\end{equation*}
The position and velocity are vector-valued, while the orientation lies on $\LieGroupSO{3}$. %
Linearizing $\pi(\Matrix{Y}_k)$ about the operating point gives the measurement perturbation
\begin{equation}
\delta \Vector{y}_k
\Defined
\bbm
\delta\Vector{r}_k^\T &
\delta\Vector{v}_k^\T &
\delta\boldsymbol{\vartheta}_k^\T
\ebm^\T
\in \Real^9,
\end{equation}
where $\delta\Vector{r}_k$ and $\delta\Vector{v}_k$ are the position and velocity perturbations.

\subsubsection{A Single Full-State Measurement}
\label{subsubsec:SGal3_one_full_measurement_jacobian}

Consider a single delayed full-state measurement at time $t_1$, consisting of position, velocity, and orientation.
The linearized measurement equation is
\begin{equation}
\delta \Vector{y}_1
\approx
\Matrix{H}_1^{(rv\phi)}\<\delta\SmallStateVector,
\end{equation}
with Jacobian
\begin{equation}
\label{eqn:single_full_jacobian}
\Matrix{H}_1^{(rv\phi)}
=
\bbm
\Identity_3
&
t_{d,1}\Identity_3
&
-\Matrix{C}_0\<\breve{\Vector{r}}(t_{d,1})^\wedge
&
-\Vector{v}(t_{d,1})
&
\Matrix{C}_0\<\Matrix{\Psi}^{(r)}_{\omega}(t_{d,1})
&
\Matrix{C}_0\<\Matrix{\Psi}^{(r)}_{a}(t_{d,1})
\\[0.5mm]
\Zero
&
\Identity_3
&
-\Matrix{C}_0\<\breve{\Vector{v}}(t_{d,1})^\wedge
&
-\Vector{a}(t_{d,1})
&
\Matrix{C}_0\<\Matrix{\Psi}^{(v)}_{\omega}(t_{d,1})
&
\Matrix{C}_0\<\Matrix{\Psi}^{(v)}_{a}(t_{d,1})
\\[0.5mm]
\Zero
&
\Zero
&
\breve{\Matrix{C}}(t_{d,1})^\T
&
-\Vector{\omega}(t_{d,1})
&
\Matrix{\Psi}^{(\phi)}_{\omega}(t_{d,1})
&
\Zero
\ebm.
\end{equation}
In the velocity row, the delay entry $-\Vector{a}(t_{d,1})$ is the coordinate acceleration of the vehicle in the inertial reference frame, $\Vector{a}(t) = \Matrix{C}(t)\<\Vector{s}(t)$, where $\Vector{s}(t)$ is the bias-corrected specific force from \Cref{eqn:SGal3_bias_corrected_inputs}.
The orientation sensitivity with respect to the gyroscope bias is defined through the local perturbation
\begin{equation*}
\delta\boldsymbol{\phi}_{\omega}(t)
\Defined
\Matlog{\breve{\Matrix{C}}(t)^{-1}\breve{\Matrix{C}}_{\omega}(t)}^\vee,
\end{equation*}
where $\breve{\Matrix{C}}_{\omega}(t)$ denotes the propagated orientation under a perturbed bias, to first order in $\delta\Vector{b}_{\omega}$.
The bias sensitivities are
\begin{align*}
\Matrix{\Psi}^{(r)}_{\omega}(t)
&\Defined
\frac{\partial\<\breve{\Vector{r}}(t)}{\partial\<\Vector{b}_{\omega}},
\qquad
\Matrix{\Psi}^{(v)}_{\omega}(t)
\Defined
\frac{\partial\<\breve{\Vector{v}}(t)}{\partial\<\Vector{b}_{\omega}},
\qquad
\Matrix{\Psi}^{(\phi)}_{\omega}(t)
\Defined
\frac{\partial\<\delta\boldsymbol{\phi}_{\omega}(t)}{\partial\<\Vector{b}_{\omega}}, \\
\Matrix{\Psi}^{(r)}_{a}(t)
&\Defined
\frac{\partial\<\breve{\Vector{r}}(t)}{\partial\<\Vector{b}_a},
\qquad
\Matrix{\Psi}^{(v)}_{a}(t)
\Defined
\frac{\partial\<\breve{\Vector{v}}(t)}{\partial\<\Vector{b}_a}.
\end{align*}
The matrices above depend on the input history and the propagated trajectory, and are not available in closed form in general.

To characterize the corresponding nullspace, set $\Matrix{H}_1^{(rv\phi)}\<\delta\SmallStateVector = \Vector{0}_{9\times 1}$.
We solve the three block rows from bottom to top, since the orientation row involves the fewest unknowns.
The three block rows give
\begin{align}
\delta\Vector{r}_0
+
t_{d,1}\<\delta\Vector{v}_0
-
\Matrix{C}_0\<\breve{\Vector{r}}(t_{d,1})^\wedge
\delta\boldsymbol{\phi}_0
-
\Vector{v}(t_{d,1})\<\delta\tau
+
\Matrix{C}_0\<\Matrix{\Psi}^{(r)}_{\omega}(t_{d,1})\<\delta\Vector{b}_{\omega}
+
\Matrix{C}_0\<\Matrix{\Psi}^{(r)}_{a}(t_{d,1})\<\delta\Vector{b}_a
& =
\Vector{0}_{3\times 1},
\label{eqn:SGal3_null_row1} \\
\delta\Vector{v}_0
-
\Matrix{C}_0\<\breve{\Vector{v}}(t_{d,1})^\wedge \delta\boldsymbol{\phi}_0
-
\Vector{a}(t_{d,1})\<\delta\tau
+
\Matrix{C}_0\<\Matrix{\Psi}^{(v)}_{\omega}(t_{d,1})\<\delta\Vector{b}_{\omega}
+
\Matrix{C}_0\<\Matrix{\Psi}^{(v)}_{a}(t_{d,1})\<\delta\Vector{b}_a
& =
\Vector{0}_{3\times 1},
\label{eqn:SGal3_null_row2}
\\
\breve{\Matrix{C}}(t_{d,1})^\T\<\delta\boldsymbol{\phi}_0
-
\Vector{\omega}(t_{d,1})\<\delta\tau
+
\Matrix{\Psi}^{(\phi)}_{\omega}(t_{d,1})\<\delta\Vector{b}_{\omega}
& =
\Vector{0}_{3\times 1}.
\label{eqn:SGal3_null_row3}
\end{align}
We isolate the delay direction by setting the bias perturbations to zero, $\delta\Vector{b}_{\omega} = \delta\Vector{b}_a = \Zero$.
The orientation row \Cref{eqn:SGal3_null_row3} then gives the initial-orientation perturbation directly,
\begin{equation}
\delta\boldsymbol{\phi}_0
=
\breve{\Matrix{C}}(t_{d,1})\<\Vector{\omega}(t_{d,1})\<\delta\tau
=
\Vector{\omega}_0(t_{d,1})\<\delta\tau,
\end{equation}
where $\Vector{\omega}_0(t) \Defined \breve{\Matrix{C}}(t)\<\Vector{\omega}(t)$ is the angular velocity expressed in the initial-vehicle frame.
With $\delta\boldsymbol{\phi}_0$ known, the velocity row \Cref{eqn:SGal3_null_row2} gives the initial-velocity perturbation,
\begin{equation}
\delta\Vector{v}_0
=
\bigl(
\Matrix{C}_0\<
\breve{\Vector{v}}(t_{d,1})^\wedge\Vector{\omega}_0(t_{d,1})
+
\Vector{a}(t_{d,1})
\bigr)\<\delta\tau,
\end{equation}
and finally the position row \Cref{eqn:SGal3_null_row1} gives the initial-position perturbation,
\begin{equation}
\delta\Vector{r}_0
=
\Big(
\Vector{v}(t_{d,1})
-
t_{d,1}
\big(
\Matrix{C}_0\<
\breve{\Vector{v}}(t_{d,1})^\wedge \Vector{\omega}_0(t_{d,1})
+
\Vector{a}(t_{d,1})
\big)
+
\Matrix{C}_0\<
\breve{\Vector{r}}(t_{d,1})^\wedge \Vector{\omega}_0(t_{d,1})
\Big)\<\delta\tau.
\end{equation}
Hence, the nullspace contains the one-dimensional direction
\begin{equation}
\Ker{\Matrix{H}_1^{(rv\phi)}}
\supseteq
\Span{
\bbm
\Vector{v}(t_{d,1})
-
t_{d,1}\big(
\Matrix{C}_0\<
\breve{\Vector{v}}(t_{d,1})^\wedge\Vector{\omega}_0(t_{d,1})
+
\Vector{a}(t_{d,1})
\big)
+
\Matrix{C}_0\<
\breve{\Vector{r}}(t_{d,1})^\wedge\Vector{\omega}_0(t_{d,1})
\\[0.5mm]
\Matrix{C}_0\<
\breve{\Vector{v}}(t_{d,1})^\wedge\Vector{\omega}_0(t_{d,1})
+
\Vector{a}(t_{d,1})
\\[0.5mm]
\Vector{\omega}_0(t_{d,1})
\\[0.5mm]
1
\\[0.5mm]
\Zero_{3\times 1}
\\[0.5mm]
\Zero_{3\times 1}
\ebm
}.
\end{equation}
This is the direct analogue of the $\LieGroupSGal{1}$ case: a change in the delay can be compensated by corresponding changes in the initial condition, leaving the measurement unchanged to first order.

\subsubsection{Two Full-State Measurements}
\label{subsubsec:SGal3_two_full_measurements_jacobian}

We now consider the case of two full-state measurements at times $t_1$ and $t_2$.
The stacked Jacobian is
\begin{equation}
\label{eqn:SGal3_two_full_stacked}
\delta \Vector{y}_{1:2}
\approx
\Matrix{H}_{1:2}^{(rv\phi)}\<
\delta\SmallStateVector,
\quad
\Matrix{H}_{1:2}^{(rv\phi)}
=
\bbm
\Matrix{H}_1^{(rv\phi)} \\[0.5mm]
\Matrix{H}_2^{(rv\phi)}
\ebm
\in \Real^{18\times 16},
\end{equation}
where each block $\Matrix{H}_k^{(rv\phi)} \in \Real^{9\times 16}$ has the form
\begin{equation}
\label{eqn:SGal3_two_full_single_block}
\Matrix{H}_k^{(rv\phi)}
=
\bbm
\Identity_3
&
t_{d,k}\Identity_3
&
-\Matrix{C}_0\<\breve{\Vector{r}}(t_{d,k})^\wedge
&
-\Vector{v}(t_{d,k})
&
\Matrix{C}_0\<\Matrix{\Psi}^{(r)}_{\omega}(t_{d,k})
&
\Matrix{C}_0\<\Matrix{\Psi}^{(r)}_{a}(t_{d,k})
\\[0.5mm]
\Zero
&
\Identity_3
&
-\Matrix{C}_0\<\breve{\Vector{v}}(t_{d,k})^\wedge
&
-\Vector{a}(t_{d,k})
&
\Matrix{C}_0\<\Matrix{\Psi}^{(v)}_{\omega}(t_{d,k})
&
\Matrix{C}_0\<\Matrix{\Psi}^{(v)}_{a}(t_{d,k})
\\[0.5mm]
\Zero
&
\Zero
&
\breve{\Matrix{C}}(t_{d,k})^\T
&
-\Vector{\omega}(t_{d,k})
&
\Matrix{\Psi}^{(\phi)}_{\omega}(t_{d,k})
&
\Zero
\ebm,
\quad
k = 1,2.
\end{equation}

This is the Jacobian counterpart of the direct algebraic argument: the two measurements jointly constrain the initial condition, the delay, and the IMU biases.
Unlike the single-measurement case, there are now $18$ scalar constraints for $16$ unknowns, so generic local identifiability is possible.
In particular, if
\begin{equation}
\Rank{\Matrix{H}^{(rv\phi)}_{1:2}} = 16,
\end{equation}
then the initial condition, the delay, and the constant IMU biases are locally identifiable.
Thus, two full-state measurements are generically sufficient for local identifiability in the $\LieGroupSGal{3}$ aided navigation problem.

\subsubsection{Two Pose Measurements}
\label{subsubsec:SGal3_two_pose_measurements_jacobian}

The two-pose case clarifies how the delay-related nullspace direction is constrained by multiple measurements.
The stacked Jacobian is
\begin{equation}
\label{eqn:SGal3_two_pose_stacked}
\delta\Vector{y}_{1:2}
\approx
\Matrix{H}_{1:2}^{(r\phi)}\<\delta\SmallStateVector,
\quad
\Matrix{H}_{1:2}^{(r\phi)}
=
\bbm
\Matrix{H}_1^{(r\phi)} \\[0.5mm]
\Matrix{H}_2^{(r\phi)}
\ebm
\in \Real^{12 \times 16},
\end{equation}
where each block $\Matrix{H}_k^{(r\phi)} \in \Real^{6\times 16}$ has the form
\begin{equation}
\label{eqn:SGal3_two_pose_single_block}
\Matrix{H}_k^{(r\phi)}
=
\bbm
\Identity_3
&
t_{d,k}\Identity_3
&
-\Matrix{C}_0\<\breve{\Vector{r}}(t_{d,k})^\wedge
&
-\Vector{v}(t_{d,k})
&
\Matrix{C}_0\<\Matrix{\Psi}^{(r)}_{\omega}(t_{d,k})
&
\Matrix{C}_0\<\Matrix{\Psi}^{(r)}_{a}(t_{d,k})
\\[0.5mm]
\Zero
&
\Zero
&
\breve{\Matrix{C}}(t_{d,k})^\T
&
-\Vector{\omega}(t_{d,k})
&
\Matrix{\Psi}^{(\phi)}_{\omega}(t_{d,k})
&
\Zero
\ebm,
\quad
k = 1, 2.
\end{equation}
To characterize the corresponding nullspace, set $\Matrix{H}_{1:2}^{(r\phi)}\<\delta\SmallStateVector = \Vector{0}_{12\times 1}$.
The two block rows give, for $k=1,2$,
\begin{align}
\delta\Vector{r}_0
+
t_{d,k}\<\delta\Vector{v}_0
-
\Matrix{C}_0\<\breve{\Vector{r}}(t_{d,k})^\wedge \delta\boldsymbol{\phi}_0
-
\Vector{v}(t_{d,k})\<\delta\tau
+
\Matrix{C}_0\<\Matrix{\Psi}^{(r)}_{\omega}(t_{d,k})\<\delta\Vector{b}_{\omega}
+
\Matrix{C}_0\<\Matrix{\Psi}^{(r)}_{a}(t_{d,k})\<\delta\Vector{b}_a
&=
\Vector{0}_{3\times 1},
\label{eqn:SGal3_two_pose_rk_rewrite}
\\
\breve{\Matrix{C}}(t_{d,k})^\T\delta\boldsymbol{\phi}_0
-
\Vector{\omega}(t_{d,k})\<\delta\tau
+
\Matrix{\Psi}^{(\phi)}_{\omega}(t_{d,k})\<\delta\Vector{b}_{\omega}
&=
\Vector{0}_{3\times 1}.
\label{eqn:SGal3_two_pose_phik_rewrite}
\end{align}

As before, we consider only the case where the gyroscope and accelerometer bias perturbations are zero.
Then \Cref{eqn:SGal3_two_pose_phik_rewrite} reduces to
\begin{equation}
\delta\boldsymbol{\phi}_0
=
\breve{\Matrix{C}}(t_{d,k})\<
\Vector{\omega}(t_{d,k})\<\delta\tau
=
\Vector{\omega}_0(t_{d,k})\<\delta\tau,
\quad
k = 1, 2,
\end{equation}
where, again, $\Vector{\omega}_0(t) \Defined \breve{\Matrix{C}}(t)\<\Vector{\omega}(t)$ is the angular velocity expressed in the initial-vehicle frame.
For a nonzero delay perturbation to remain in the nullspace, these two expressions must agree, so
\begin{equation}
\label{eqn:SGal3_two_pose_time_condition}
\Vector{\omega}_0(t_{d,1})=\Vector{\omega}_0(t_{d,2}).
\end{equation}
Under this condition,
\begin{equation}
\delta\boldsymbol{\phi}_0
=
\Vector{\omega}_0(t_{d,1})\<\delta\tau.
\label{eqn:SGal3_two_pose_phi_tau}
\end{equation}
Next, subtracting \Cref{eqn:SGal3_two_pose_rk_rewrite} for $k=1$ and $k=2$ eliminates $\delta\Vector{r}_0$ and gives
\begin{equation}
(t_2 - t_1)\<\delta\Vector{v}_0
-
\Matrix{C}_0
\Bigl(
\breve{\Vector{r}}(t_{d,2})^\wedge-\breve{\Vector{r}}(t_{d,1})^\wedge
\Bigr)\delta\boldsymbol{\phi}_0
-
\bigl(\Vector{v}(t_{d,2})-\Vector{v}(t_{d,1})\bigr)\<\delta\tau
=
\Vector{0}_{3\times 1}.
\end{equation}
Substituting \Cref{eqn:SGal3_two_pose_phi_tau} yields
\begin{equation}
\label{eqn:SGal3_two_pose_v_tau}
\delta\Vector{v}_0
=
\Vector{d}_v\<\delta\tau,
\end{equation}
where
\begin{equation}
\Vector{d}_v
\Defined
\frac{1}{t_2 - t_1}\<
\Matrix{C}_0
\Bigl(
\breve{\Vector{r}}(t_{d,2})^\wedge - \breve{\Vector{r}}(t_{d,1})^\wedge
\Bigr)\Vector{\omega}_0(t_{d,1})
+
\frac{1}{t_2 - t_1}\<
\bigl(\Vector{v}(t_{d,2}) - \Vector{v}(t_{d,1})\bigr).
\end{equation}
Finally, substituting \Cref{eqn:SGal3_two_pose_phi_tau} and \Cref{eqn:SGal3_two_pose_v_tau} into \Cref{eqn:SGal3_two_pose_rk_rewrite} at $k = 1$ gives
\begin{equation}
\delta\Vector{r}_0
=
\Vector{d}_r\<\delta\tau,
\label{eqn:SGal3_two_pose_r_tau}
\end{equation}
where
\begin{equation}
\Vector{d}_r
\Defined
-t_{d,1}\<\Vector{d}_v +
\Matrix{C}_0\<\breve{\Vector{r}}(t_{d,1})^\wedge \Vector{\omega}_0(t_{d,1}) +
\Vector{v}(t_{d,1}).
\end{equation}
Therefore, when \Cref{eqn:SGal3_two_pose_time_condition} holds,
\begin{equation}
\Ker{\Matrix{H}_{1:2}^{(r\phi)}}
\supseteq
\Span{
\bbm
\Vector{d}_r \\
\Vector{d}_v \\
\Vector{\omega}_0(t_{d,1}) \\
1 \\
\Zero_{3\times 1} \\
\Zero_{3\times 1}
\ebm
}.
\end{equation}
This is the same basic situation seen in the $\LieGroupSGal{1}$ case: a change in the delay can be absorbed by corresponding changes in the initial condition, with the biases held fixed.
In three dimensions, the compensation also includes an orientation perturbation, and this nullspace direction exists only when the angular velocities at the two delayed measurement times, expressed in the initial-vehicle frame, are equal.

\subsection{Uninformative Trajectories}
\label{subsec:SGal3_bad_trajectories}

Identifiability of the delay (and the initial condition) depends not only on the number of measurements, but also on whether the input is sufficiently exciting.
As in the one-dimensional case, there exist trajectories for which the delay and initial condition cannot be uniquely determined, regardless of the number or type of measurements.
We again refer to such trajectories as \emph{uninformative}.

The relevant class of uninformative trajectories is characterized by motion that is generated by a constant Lie algebra element from the initial time onward.
We hold the biases fixed to isolate the delay–initial-condition symmetry, working directly with the resulting trajectory.
Suppose that, from the initial time $t = 0$ onward, the bias-corrected angular velocity and specific force are constant,
\begin{equation}
\Vector{\omega}(t) = \Vector{\omega}_0,
\quad
\Vector{s}(t) = \Vector{s}_0,
\end{equation}
for $\Vector{\omega}_0, \Vector{s}_0 \in \Real^3$, as in \Cref{subsec:SGal1_bad_trajectories}.%
\footnote{Under constant vehicle-frame angular velocity, the initial-frame rate $\Vector{\omega}_0(t) = \breve{\Matrix{C}}(t)\<\Vector{\omega}_0$ of \Cref{subsec:SGal3_jacobian_nullspace} reduces to the constant $\Vector{\omega}_0$, since the rotation is about its own axis.}
In this case, the state-transition matrix $\Matrix{\Phi}(t)$ is generated by the exponential map of a fixed Lie algebra element.
Because $\Matrix{\Phi}(0) = \Identity_5$, the generator is not arbitrary: it must be chosen to match the zero initial conditions built into $\Matrix{\Phi}(t)$.
Specifically,
\begin{equation}
\Matrix{\Phi}(t)
=
\Matexp{t\<\Vector{\xi}^{\wedge}_{c}},
\quad
\Vector{\xi}^{\wedge}_{c}
=
\bbm
\Vector{\omega}_0^{\wedge} & \Vector{s}_0 & \Zero \\
\Zero & 0 & 1 \\
\Zero & 0 & 0
\ebm.
\end{equation}
It follows that
\begin{equation}
\Matrix{\Phi}(t_k - \tau)\<
\Matrix{\Phi}(t_k - \alttau)^{-1}
=
\Matexp{(\alttau - \tau)\<
\Vector{\xi}^{\wedge}_{c}},
\end{equation}
which is independent of $t_k$.

Returning to the measurement model, the same transformation used in the single-measurement analysis therefore applies simultaneously to all measurements.
Define
\begin{equation}
\bar{\Matrix{X}}_0
=
\Matrix{T}(\tau - \alttau)\<
\Matrix{X}_0\<
\Matrix{\Phi}(\alttau - \tau),
\end{equation}
where the left time translation keeps $\bar{\Matrix{X}}_0$ in the isochronous subgroup.
Then, for all measurement times $t_k$,
\begin{equation}
\Matrix{Y}_k
=
\Matrix{T}(\tau)\<
\Matrix{X}_0\<
\Matrix{\Phi}(t_k - \tau)
=
\Matrix{T}(\alttau)\<
\bar{\Matrix{X}}_0\<
\Matrix{\Phi}(t_k - \alttau).
\end{equation}
The entire measurement set can therefore be explained by a different value of $\tau$ together with a corresponding change in the initial condition.

To connect this with \Cref{subsec:symmetries}, consider the parameter tuple
\begin{equation*}
\StateVector
\Defined
(\Matrix{X}_0, \tau).
\end{equation*}
As before, assume that $\tau$ lies in the interior of the admissible delay
interval, and choose $\epsilon>0$ sufficiently small that $\tau + \alpha$ and all shifted delayed evaluation times remain admissible for every $\alpha\in(-\epsilon,\epsilon)$.
Define a transformation $\mathcal{S}_\alpha$ on this parameter tuple by
\begin{equation*}
\mathcal{S}_\alpha :
(\Matrix{X}_0, \tau)
\mapsto
\bigl(
\Matrix{T}(-\alpha)\<
\Matrix{X}_0\<
\Matrix{\Phi}(\alpha),\,
\tau + \alpha
\bigr),
\end{equation*}
where the leading time translation keeps the transformed initial condition in the isochronous subgroup.
For motions of the form considered here, the parameter tuples $(\Matrix{X}_0, \tau)$ and $\mathcal{S}_\alpha\big((\Matrix{X}_0,\tau)\big)$ generate the same measurement history for all sufficiently small $\alpha$.
Therefore, \Cref{cor:symmetry_unidentifiable_discrete} applies (and likewise \Cref{thm:symmetry_unidentifiable_continuous}), and local identifiability fails.

\subsection{Dealing With Gravity}
\label{subsec:SGal3_now_with_gravity}

As a final step, we add gravity to the aided navigation model from \Cref{subsec:SGal3_system}.
The inclusion of gravity does not alter the conclusions of \Cref{subsec:SGal3_identifiability}: the measurement requirements are unchanged, and the excitation requirements, expressed in terms of the bias-corrected inputs, are identical.
What changes is the \emph{shape} of the uninformative trajectories, through a known, gravity-dependent factor in the trajectory and in the symmetry.

The continuous-time kinematics are now
\begin{align}
\dot{\Matrix{C}}(t)
& =
\Matrix{C}(t)\bigl(\Vector{\omega}_m(t) - \Vector{b}_{\omega}\bigr)^{\wedge}, \\
\label{eqn:gravity_added_force}
\dot{\Vector{v}}(t)
& =
\Matrix{C}(t)\bigl(\Vector{a}_m(t) - \Vector{b}_a\bigr) + \Vector{g}, \\
\dot{\Vector{r}}(t)
& =
\Vector{v}(t).
\end{align}
The measured specific force $\Vector{a}_m(t)$ is unchanged, but its relation to the coordinate acceleration is now $\Vector{a}(t) = \Matrix{C}(t)\<\Vector{s}(t) + \Vector{g}$.
The propagated terms $\breve{\Matrix{C}}(t)$, $\breve{\Vector{v}}(t)$, and $\breve{\Vector{r}}(t)$ satisfy the same equations, \Crefrange{eqn:SGal3_C_breve_dot}{eqn:SGal3_r_breve_dot}, and are still determined by the bias-corrected IMU history.
However, $\breve{\Vector{v}}(t)$ and $\breve{\Vector{r}}(t)$ now incorporate an integrated component opposite $\Vector{g}$, which is added back to give the trajectory,
\begin{align*}
\Matrix{C}(t)
&=
\Matrix{C}_0\<\breve{\Matrix{C}}(t), \\
\Vector{v}(t)
&=
\Matrix{C}_0\<\breve{\Vector{v}}(t) + \Vector{v}_0 + t\<\Vector{g}, \\
\Vector{r}(t)
&=
\Matrix{C}_0\<\breve{\Vector{r}}(t) + t\<\Vector{v}_0 + \Vector{r}_0 + \tfrac{1}{2}\<t^2\<\Vector{g}.
\end{align*}
These additive terms are generated by a single element of the Lie algebra.
Define the \emph{gravity generator}
\begin{equation}
\label{eqn:sgal3_gravity_generator}
\Vector{\xi}_{g}
\Defined
\bbm
\Zero_{1 \times 3} & \Vector{g}^\T & \Zero_{1 \times 3} & 1
\ebm^\T
\in \Real^{10},
\end{equation}
with the gravity vector in the boost slot, and let
\begin{equation}
\label{eqn:SGal3_free_fall_subgroup}
\Matrix{W}(\beta)
\Defined
\Matexp{\beta\<\Vector{\xi}_{g}^{\wedge}}
=
\bbm
\Identity_3 & \beta\<\Vector{g} & \tfrac{1}{2}\<\beta^2\Vector{g} \\
\Zero & 1 & \beta \\
\Zero & 0 & 1
\ebm
\end{equation}
denote the corresponding one-parameter subgroup, where the closed form follows from \Cref{eqn:SGal3_exp_closed} with $\Vector{\phi} = \Zero$.
As in the $\LieGroupSGal{1}$ case, the quadratic position term reflects the boost--time bracket $[\Matrix{B},\< \Matrix{K}] = \Matrix{P}$.
The contribution of gravity to the trajectory over the interval $[0, t]$ is then the product of the two one-parameter subgroups,
\begin{equation}
\label{eqn:SGal3_gravity_factor}
\Matrix{G}(t)
\Defined
\Matrix{W}(t)\<\Matrix{T}(-t)
=
\bbm
\Identity_3 & t\<\Vector{g} & -\tfrac{1}{2}\<t^2\Vector{g} \\
\Zero & 1 & 0 \\
\Zero & 0 & 1
\ebm,
\end{equation}
which lies in the isochronous subgroup.\footnote{Note that $\Matrix{G}(t)$ is not itself a one-parameter subgroup.}
The trajectory is then
\begin{equation}
\label{eqn:SGal3_gravity_state}
\Matrix{X}(t) = \Matrix{G}(t)\<\Matrix{X}_0\<\Matrix{\Phi}(t),
\end{equation}
where $\Matrix{X}_0$ is isochronous and $\Matrix{\Phi}(t)$ is the usual state-transition matrix, which contributes the group time coordinate of $\Matrix{X}(t)$.

\begin{infobox}[t!]
\begin{mdframed}
\textbf{Which Trajectories Are Uninformative?}
\vspace{0.5\baselineskip}
\par\noindent Several familiar trajectory shapes turn out to be uninformative.
Each is generated by a constant pair $(\Vector{\omega}_0, \Vector{s}_0)$ of bias-corrected inputs; a few concrete cases are listed below.
\begin{itemize}
\setlength{\itemsep}{2pt}
\item Ballistic motion, produced by zero specific force and constant (possibly zero) vehicle-frame angular velocity.
The vehicle is in free fall, tumbling at a constant rate, and the trajectory is a parabola.

\item Constant-thrust motion, produced by zero angular velocity and constant nonzero specific force.
The coordinate acceleration is constant, and the trajectory is in general a parabolic arc.

\item Circular or helical motion, produced by constant vehicle-frame angular velocity and constant specific force.
A coordinated turn at constant bank angle is an example.

\item General curvilinear motion, produced by any other constant pair $(\Vector{\omega}_0, \Vector{s}_0)$.
The vehicle rotates at a constant rate while the vehicle-frame velocity varies, giving a curved path.
\end{itemize}
What unifies these cases is not any particular geometry but the input condition: the bias-corrected angular velocity and specific force are constant from the initial time through the delayed measurement times, so the delayed measurement history cannot distinguish a shift of the delay from a compensating shift of the initial condition.
\end{mdframed}
\vspace{-\baselineskip}
\end{infobox}

The delayed measurement model follows that in \Cref{eqn:SGal3_delayed_measurement}, but with gravity now included,
\begin{equation}
\label{eqn:SGal3_delayed_measurement_gravity}
\Matrix{Y}_k
\Defined
\Matrix{T}(\tau)\<
\Matrix{X}(t_k - \tau)
=
\Matrix{T}(\tau)\<
\Matrix{G}(t_k - \tau)\<
\Matrix{X}_0\<
\Matrix{\Phi}(t_k - \tau),
\end{equation}
and has components
\begin{align*}
\Matrix{C}(t_k - \tau)
& =
\Matrix{C}_0\<\breve{\Matrix{C}}(t_k - \tau), \\
\Vector{v}(t_k - \tau)
& =
\Matrix{C}_0\<\breve{\Vector{v}}(t_k - \tau) + 
\Vector{v}_0 + (t_k - \tau)\<\Vector{g}, \\
\Vector{r}(t_k - \tau)
& =
\Matrix{C}_0\<\breve{\Vector{r}}(t_k - \tau) + 
(t_k - \tau)\<\Vector{v}_0 + 
\Vector{r}_0 + 
\tfrac{1}{2}\<(t_k - \tau)^2\Vector{g},
\end{align*}
with group time coordinate $t_k$.

The measurement Jacobian retains the form of \Cref{eqn:single_full_jacobian}.
Gravity affects only the delay column: the entries $\Vector{v}(t_{d,k})$ and $\Vector{a}(t_{d,k})$ now include the contributions $t_{d,k}\<\Vector{g}$ and $\Vector{g}$, respectively, since $\Vector{a}(t) = \Matrix{C}(t)\<\Vector{s}(t) + \Vector{g}$.
Because the parameter dependence of the Jacobian is identical, the measurement-count and rank analysis conclusions of \Cref{subsec:SGal3_identifiability} carry over without any changes.

Finally, we revisit the class of uninformative trajectories.
Assume, as in \Cref{subsec:SGal3_bad_trajectories}, that the bias-corrected angular velocity and specific force are constant, $\Vector{\omega}(t) = \Vector{\omega}_0$ and $\Vector{s}(t) = \Vector{s}_0$, so that $\Matrix{\Phi}(t) = \Matexp{t\<\Vector{\xi}^{\wedge}_{c}}$ with constant generator $\Vector{\xi}_{c}$.
For $\alpha$ near zero, define
\begin{equation}
\label{eqn:SGal3_gravity_symmetry}
\bar{\Matrix{X}}_0
\Defined
\Matrix{T}(-\alpha)\<\Matrix{G}(\alpha)\<
\Matrix{X}_0\<
\Matrix{\Phi}(\alpha),
\quad
\bar{\tau} \Defined \tau + \alpha.
\end{equation}
The transformed initial condition is isochronous, since the time components of $\Matrix{T}(-\alpha)$ and $\Matrix{\Phi}(\alpha)$ cancel out.
Multiplying out the factors gives the composition identity
\begin{equation}
\label{eqn:SGal3_gravity_closure}
\Matrix{T}(\alpha)\<\Matrix{G}(u)\<\Matrix{T}(-\alpha)\<\Matrix{G}(\alpha)
=
\Matrix{G}(u + \alpha),
\end{equation}
and hence, with $u = t_k - \tau - \alpha$,
\begin{align*}
\bar{\Matrix{Y}}_k
& =
\Matrix{T}(\bar{\tau})\<\Matrix{G}(u)\<\bar{\Matrix{X}}_0\<\Matrix{\Phi}(u) \\
& =
\Matrix{T}(\tau)\<
\bigl[\Matrix{T}(\alpha)\<\Matrix{G}(u)\<\Matrix{T}(-\alpha)\<\Matrix{G}(\alpha)\bigr]\<
\Matrix{X}_0\<
\Matrix{\Phi}(\alpha)\<\Matrix{\Phi}(u) \\
& =
\Matrix{T}(\tau)\<\Matrix{G}(t_k - \tau)\<\Matrix{X}_0\<\Matrix{\Phi}(t_k - \tau)
=
\Matrix{Y}_k,
\end{align*}
so the delayed measurement history is unchanged and the delay and initial condition remain jointly unidentifiable.
The symmetry differs from the gravity-free case of \Cref{subsec:SGal3_bad_trajectories} only by the known factor $\Matrix{G}(\alpha)$, which accounts for the contribution of gravity over the shifted interval.
In fact, the combined left factor is itself a one-parameter subgroup,
\begin{equation}
\label{eqn:SGal3_gravity_left_generator}
\Matrix{T}(-\alpha)\<\Matrix{G}(\alpha)
=
\Matexp{\alpha\<\Vector{\xi}_{g^-}^{\wedge}},
\quad
\Vector{\xi}_{g^-}
\Defined
\bbm
\Zero_{1 \times 3} & \Vector{g}^\T & \Zero_{1 \times 3} & -1
\ebm^\T,
\end{equation}
so gravity enters the symmetry only through the boost slot of the left generator.%
\footnote{The same group element appears as the gravity factor in equivariant IMU preintegration~\cite{2025_Delama_Equivariant}.}
Setting $\Vector{g} = \Zero$ gives back $-\Vector{\xi}_{d}$ and the gravity-free result.

\section{Conclusion}
\label{sec:conclusion}

In this report, we studied delay identifiability in aided inertial navigation using the special Galilean group.
We showed that identifiability depends on both the available measurements and the form of the trajectory, characterizing when the delay and initial condition can be recovered from full-state, pose, and position-only measurements. The same conclusions follow from a direct constraint-based analysis and from a Jacobian-based analysis.
Gravity fits naturally within the framework also: it enters the trajectory, and the symmetry, through a known group factor, leaving the measurement and excitation requirements unchanged.

A central theme throughout is that identifiability depends not only on the number and type of measurements, but also on the motion itself.
This dependence leads to a characterization of uninformative trajectories in terms of symmetry.
When the bias-corrected inputs are constant, the state-transition matrix traces a one-parameter subgroup of the special Galilean group, and shifts in the delay can be absorbed into corresponding changes in the initial condition without altering the measurement history.
The use of Lie algebra generators to characterize uninformative trajectories on the special Galilean group is, to our knowledge, new.

\appendix

\crefalias{section}{appendix}

\titleformat{\section}{\normalfont\Large\bfseries}{Appendix~\thesection}{0.75em}{}

\section{Revision History}
\label{app:revisions}

A rough list of revisions to the report follows.

\begin{itemize}
\setlength{\itemsep}{2pt}
\item Revision 1.01, 2026-02-03 --- Initial release.
\item Revision 1.02, 2026-02-15 --- Added further details on continuous symmetries.
\item Revision 1.03, 2026-03-19 --- Split $\LieGroupSGal{3}$ analysis into component form (position, orientation, etc.).
\item Revision 1.04, 2026-03-30 --- Updated $\LieGroupSGal{3}$ results to include six position measurements.
\item Revision 1.05, 2026-04-22 --- Added short $\LieGroupSGal{1}$ filtering discussion.
\item Revision 1.06, 2026-05-20 --- Reformatted $\LieGroupSGal{3}$ Jacobians.
\item Revision 1.07, 2026-05-30 --- Fixed various typos throughout.
\item Revision 1.08, 2026-06-29 --- Clarified (corrected) gravity-free model in \Cref{subsec:SGal3_system}.
\item Revision 1.09, 2026-07-01 --- Fixed commentary on uninformative trajectories.
\item Revision 1.10, 2026-07-12 --- Added treatment of gravity (as part of the symmetry) in \Cref{subsec:SGal3_now_with_gravity}.
\item Revision 1.11, 2026-08-01 --- Minor clarifications throughout (sufficiency versus necessity, etc).
\end{itemize}

\printbibliography

\end{document}